\documentclass[a4,a4wide]{article}

\usepackage{aaai}
\usepackage{times}
\usepackage{helvet}
\usepackage{courier}

\usepackage{amsmath}
\usepackage{amssymb}
\usepackage{amsthm}
\usepackage{url}
\usepackage{complexity}

\usepackage{tikz}
\usetikzlibrary{shapes,decorations,shadows}

\newtheorem{theorem}{Theorem}
\newtheorem{lemma}[theorem]{Lemma}
\newtheorem{proposition}[theorem]{Proposition}
\newtheorem{corollary}[theorem]{Corollary}

\newtheorem{longtheorem}{Theorem}
\newtheorem{longlemma}[longtheorem]{Lemma}

\newtheorem{longcorollary}[longtheorem]{Corollary}

\theoremstyle{definition}
\newtheorem{definition}{Definition}
\newtheorem{example}{Example}

\newtheorem{longdefinition}{Definition}
\newtheorem{longexample}{Example}

\newcommand{\N}{\mathbb{N}}
\newcommand{\commnt}[1]{}

\def\And{\wedge}                
\def\AND{\bigwedge}             
\def\Or{\vee}                   

\def\PC{\mathcal P}

\def\KC{\mathcal K}
\def\KCt{\mathcal{K}_{\geq 2}}

\def\KC{\mathcal K}

\def\And{\wedge}                
\def\AND{\bigwedge}             
\def\Or{\vee}                   
\def\OR{\bigvee}                
\def\sm{\setminus}

\def\red#1#2{#1{\downarrow}_{#2}}
\def\max{\text{max}}

\tikzstyle{arg}=[draw, thick, circle, fill=gray!15,inner sep=3pt]

\newcommand{\Args}{\textit{Arg}}
\newcommand{\Pairs}{\textit{Pairs}}
\def\noPairs#1{\overline{\textit{Pairs}_#1}}
\def\noPairsCl#1{\left(\overline{\textit{Pairs}_#1}\right)^*}

\newcommand{\cf}{\textit{cf}}
\newcommand{\adm}{\textit{adm}}

\newcommand{\naive}{\textit{naive}}
\newcommand{\stb}{\textit{stb}}
\newcommand{\prf}{\textit{pref}}
\newcommand{\pref}{\textit{pref}}
\newcommand{\stage}{\textit{stage}}
\newcommand{\sem}{\textit{sem}}
\newcommand{\semi}{\textit{sem}}

\newcommand{\CAF}{\textit{CAF}}

\newcommand{\AF}{\textit{argumentation framework}}

\newcommand{\Aa}{\mathfrak{A}}
\newcommand{\Ss}{\mathbb{S}}
\newcommand{\Tt}{\mathbb{T}}
\newcommand{\Ee}{\mathbb{E}}

\newcommand{\RX}{\mathfrak{R}}

\newcommand{\ff}{\mathfrak{f}}

\newcommand{\at}{\mapsto}

\newcommand{\card}[1]{\left|#1\right|}

\newcommand{\dcl}{\textit{dcl}}

\newcommand{\irr}{\textit{irr}}
\newcommand{\sym}{\textit{sym}}

\newcommand{\und}{\textit{und}}
\newcommand{\mis}{\textit{MIS}}

\newcommand{\allafs}{\textit{AF}_{\Aa}}

\newcommand{\af}{{\sc af}}
\newcommand{\afs}{{\sc af}s}

\newcommand{\sigmamax}[1]{\sigma_{\max}^{#1}}
\newcommand{\sigmamaxcon}{\sigmamax{\text{con}}}

\usepackage{relsize}

\newcommand{\presection}{\vspace*{-5mm}}
\newcommand{\mysection}[1]{\iflong{\section{#1}}\else{{\presection\section{#1}}}\fi}
\newcommand{\presubsection}{\vspace*{-4mm}}
\newcommand{\mysubsection}[1]{\iflong{\subsection{#1}}\else{{\presubsection\subsection{#1}}}\fi}

\usepackage{ifthen}
\usepackage{environ}
\newif\iflong
\NewEnviron{shortproof}{\iflong\else\expandafter\begin{proof}\BODY\end{proof}\fi}
\NewEnviron{longproof}{\iflong\expandafter\begin{proof}\BODY\end{proof}\fi}

\newcommand{\pregather}{\vspace*{-2mm}}
\newcommand{\postgather}{\vspace*{-2mm}}
\NewEnviron{mygather}{\iflong{\expandafter\begin{gather*}\BODY\end{gather*}}\else{\expandafter{\pregather\begin{center}\begin{math}\displaystyle\BODY\end{math}\end{center}\postgather}}\fi}

\newcommand{\lvspace}[1]{{\iflong\else\vspace*{#1}\fi}}

\usepackage[pdftex,colorlinks=true,linkcolor=black,citecolor=black,menucolor=black,urlcolor=black]{hyperref}

\usepackage{ifthen}

\newif\ifnmr 

\NewEnviron{long-and-not-nmr-proof}{\iflong\ifnmr\else\expandafter\begin{proof}\BODY\end{proof}\fi\fi}
\NewEnviron{short-or-nmr-proof}{\ifthenelse{\(\(\NOT\boolean{long}\)\OR\boolean{nmr}\)}{\expandafter\begin{proof}\BODY\end{proof}}{}}

\NewEnviron{mylongexample}{\expandafter\ifnmr{\begin{example}\BODY\end{example}}\else{\begin{longexample}\BODY\end{longexample}}\fi}
\NewEnviron{mylongdefinition}{\expandafter\ifnmr{\begin{definition}\BODY\end{definition}}\else{\begin{longdefinition}\BODY\end{longdefinition}}\fi}
\NewEnviron{mylonglemma}{\expandafter\ifnmr{\begin{lemma}\BODY\end{lemma}}\else{\begin{longlemma}\BODY\end{longlemma}}\fi}

\nmrtrue
\longtrue

\begin{document}

\nocopyright

\title{Compact Argumentation Frameworks\thanks{This research has been supported by DFG (project BR~1817/7-1) and FWF (projects I1102 and P25518).}}

\author{
Ringo Baumann \and Hannes Strass\\
Leipzig University, Germany 
\AND
Wolfgang Dvo\v{r}{\'a}k\\
\smaller University of Vienna, Austria 
\AND 
Thomas Linsbichler 
\and Stefan Woltran\\
\smaller Vienna University of Technology, Austria 
}

\maketitle

\begin{abstract}
Abstract argumentation frameworks (AFs) are one of the most studied
formalisms in AI.  In this work, we introduce a certain subclass of AFs
which we call compact. Given an extension-based semantics, 
the corresponding compact AFs are characterized by the feature 
that each argument of the AF occurs in at least one extension.
This not only guarantees a certain notion of fairness; 
compact AFs are thus also minimal in the sense that no argument can be 
removed without changing the outcome.
We address the following questions in the paper: 
(1) How are the classes of compact AFs related for different semantics?
(2) Under which circumstances can AFs be transformed into equivalent compact ones?
(3) Finally, we show that compact AFs are indeed a non-trivial subclass, 
since the verification problem remains $\coNP$-hard for certain semantics.
\end{abstract}

\lvspace{2mm}

\mysection{Introduction}
\label{sec:intro}
\lvspace{-2mm}
In recent years, \emph{argumentation} has become a major concept
in AI research \cite{DBLP:journals/ai/Bench-CaponD07,argu-book}.
In particular, Dung's well-studied \emph{abstract argumentation frameworks} (\afs) \cite{dung95}
are a simple, yet powerful formalism for modeling and deciding argumentation problems.
Over the years, various \emph{semantics} have been proposed,
which may yield different results (so called \emph{extensions})
when evaluating an \af\ \cite{dung95,Verheij96,CaminadaCD12,baro11}.
Also, some subclasses of \afs\, such as
acyclic, symmetric, odd-cycle-free or bipartite \afs,
have been considered,
where for some of these classes different semantics collapse~\cite{Coste-MarquisDM05,Dunne07}.

In this work we introduce a further class, which
to the best of our knowledge has not received attention in the literature,
albeit the idea is simple.
We will call an \af\ \emph{compact}
(with respect to a semantics $\sigma$),
if each of its arguments appears in at least one extension under $\sigma$.
Thus, compact \afs\ yield a ``semantic'' subclass since its
definition is based on the notion of extensions.
Another example of such a semantic subclass are coherent \afs\ \cite{DunneB02};
there are further examples in~\cite{agree,DvorakJWW14}.

Importance of compact {\sc af}s mainly stems from the following two aspects.
\ifnmr{
First, compact \afs\ possess a certain fairness behavior
in the sense that each argument has the chance to be accepted.
This might be a desired feature in some of the application areas
such as decision support \cite{AmgoudDM08},
where \afs\ are employed for a comparative evaluation of different options.
Given that each argument appears in some extension
ensures that the model is well-formed in the sense that
it does not contain impossible options.
}
\fi
The second and more concrete aspect is the issue of normal-forms of \afs.
Indeed, compact \afs\ are attractive for such a normal-form,
since none of the arguments can be removed without changing the extensions.

Following this idea we are interested in the question
whether an arbitrary \af\ can be transformed into a compact \af\
without changing the outcome under the considered semantics.
It is rather easy to see that under the \emph{naive} semantics,
which is defined as maximal conflict-free sets,
any \af\ can be transformed into an equivalent compact \af.
However, as has already been observed by \citeauthor{DunneDLW13}~(\citeyear{DunneDLW13}),
this is not true for other semantics.
As an example consider the following \af\ $F_1$,
where nodes represent arguments and directed edges represent attacks.
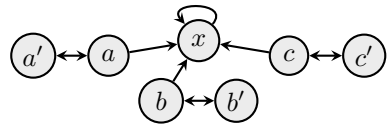
\begin{center}
\iflong{
%
\begin{tikzpicture}[scale=1,>=stealth]
 
		\path   (1,0)   node[arg](x){$x$}
			++(-1.2,-0.2)  node[arg](a){$a$}
			++(-1,0) node[arg, inner sep=2pt](aa){$a'$}
			++(1.7,-0.6) node[arg](b){$b$}
			++(1,0) node[arg, inner sep=2pt](bb){$b'$}
			++(0.7,+0.6) node[arg](c){$c$}
			++(1,0) node[arg, inner sep=2pt](cc){$c'$}
			;
                \path [<->, thick] 
		        (a) edge (aa)
			(b) edge (bb)
                        (c) edge (cc)
			 ;
                \path [->, thick]
                        (a) edge (x)
                        (b) edge (x)
                        (c) edge (x)
                        ;
                \draw[loop left,thick, distance=0.4cm, out=50, in=130,->] (x) edge (x);
\end{tikzpicture}
}\fi
\end{center}
The \emph{stable} extensions (conflict-free sets attacking all other arguments)
of $F_1$ are $\{a,b,c\}$,
$\{a,b',c'\}$,
$\{a',b,c'\}$,
$\{a',b',c\}$,
$\{a,b,c'\}$,
$\{a',b,c\}$, and
$\{a,b',c\}$.
It was shown in \cite{DunneDLW13} that there is no compact \af\
(in this case an $F_1'$ not using argument $x$)
which yields the same stable extensions as $F_1$.
By the necessity of conflict-freeness any such compact \af\
would only allow conflicts between arguments
$a$ and $a'$, $b$ and $b'$, and $c$ and $c'$, respectively.
Moreover, there must be attacks in both directions for each of these conflicts
in order to ensure stability.
Hence any compact \af\ having the same stable extensions as $F_1$
necessarily yields $\{a',b',c'\}$ in addition.
As we will see,
all semantics under consideration share certain criteria
which guarantee impossibility of a translation to a compact \af.

Like other subclasses, compact \afs\ decrease complexity of certain decision
problems. This is obvious by the definition for credulous acceptance (does an 
argument occur in at least one extension). For skeptical acceptance (does an 
argument $a$ occur in all extensions) in compact \afs\ this problem reduces to checking
whether $a$ is isolated. If yes, it is skeptically accepted; if no, $a$ is 
connected to at least one further argument which has to be credulously accepted
by the definition of compact \afs. But then, it is the case for any semantics 
which is based on conflict-free sets
that $a$ cannot be skeptically accepted, since it will not appear together with $b$ in 
an extension.
However, as we will see, the problem of verification (does a given set of arguments form an extension)
remains $\coNP$-hard for certain semantics,
hence enumerating all extensions of an \af\ remains non-trivial.

An exact characterization of the collection of all sets of extensions
which can be achieved by a compact \af\ under a given semantics $\sigma$
seems rather challenging.
We illustrate this on the example of stable semantics.
Interestingly, we can provide an exact characterization under the condition that
a certain conjecture holds:
Given an \af\ $F$ and two arguments
which do not appear jointly in an extension of $F$,
one can always add an attack between these two arguments
(and potentially adapt other attacks in the \af)
without changing the stable extensions.
This conjecture is important for our work, but also an interesting question in and of itself.

To summarize, the main contributions of our work are:
\lvspace{-2mm}
\begin{itemize}
\item
We define the classes of compact \afs\ for some of the most prominent
semantics (namely naive, stable, stage, semi-stable and preferred)
and provide a full picture of the relations between these classes.
Then we show that
the verification problem is still intractable for stage, semi-stable and preferred semantics.
\item
Moreover we use and extend recent results on maximal numbers of extensions \cite{BaumannS13}
to give some impossibility-results for \emph{compact realizability}.
That is, we provide conditions under which for an \af\ with a certain number of extensions
no translation to an equivalent (in terms of extensions) compact \af\ exists.
\item
Finally, we study \emph{signatures}
   \cite{DunneDLW14}
for compact  \af s
exemplified on the stable semantics.
An exact characterization relies on the open explicit-conflict conjecture
mentioned above. However,
we give some sufficient conditions for an extension-set to be expressed
as a stable-compact \af.
For example, 
it holds
that any \af\ with at most
three stable extensions possesses an equivalent compact \af.
\end{itemize}

\iflong{\ifnmr{}\else{
The ambivalent numbering in this version of the paper
aims at maximizing convenience for readers of the short version.
}\fi}
\fi

\lvspace{2mm}
\mysection{Preliminaries}
\label{sec:prelim}

\ifnmr{
In what follows, we recall the necessary background on 
abstract argumentation. For an excellent overview, we refer to
\cite{baro11}.
}
\fi

Throughout the paper we assume a countably infinite domain $\Aa$ of arguments.
An \AF\ ({\sc af}) is a pair $F=(A,R)$ where $A\subseteq \Aa$ is a non-empty, finite set of arguments 
and $R \subseteq A \times A$ is the attack relation. 
The collection of all {\sc af}s is given as $\allafs$.
For an \af\ $F=(B,S)$ we use $A_F$ and $R_F$ to refer to $B$ and $S$, respectively.
We write $a \at_F b$ 
for $(a,b) \in R_F$ and 
$S \at_F a$ 
(resp.\ $a \at_F S$)
if $\exists s \in S$ such that $s \at_F a$ (resp.\ 
$a \at_F s$).
For $S \subseteq A$, the \emph{range} of $S$ (wrt.\ $F$), denoted $S_F^+$, is the set 
$S \cup \{b \mid S \at_F b\}$. 
%
%

Given $F=(A,R)$, 
an argument $a \in A$ is \emph{defended} (in $F$) by 
$S \subseteq A$ if for each $b \in A$, 
such that $b\at_F a$, also $S \at_F b$.
A set $T$ of arguments is defended (in $F$) by $S$ if each $a\in T$ is 
defended by $S$ (in $F$).
%
A set $S \subseteq A$ is \emph{conflict-free} (in $F$), if there are no arguments $a,b \in S$, 
such that $(a,b) \in R$.
\ifnmr{
$\cf(F)$ denotes the set of all conflict-free sets in $F$.
}
\fi
$S\in\cf(F)$ is called \emph{admissible} (in $F$) if $S$ defends itself.
\ifnmr{
$\adm(F)$ denotes the set of admissible sets in $F$.
}
\fi

The semantics we study in this work are the naive, stable, 
preferred, stage, and semi-stable extensions. Given $F=(A,R)$ they are 
defined as subsets of $\cf(F)$ as follows:
%
\ifnmr{
\begin{itemize}
	\item $S \in \naive (F)$, if there is no $T \in \cf (F)$ with $T \supset S$
	\item $S \in \stb(F)$, if 
		$S\at_F a$ for all $a\in A\sm S$
	\item $S \in \prf(F)$, if $S \in \adm(F)$ and 
		$\nexists T \in \adm(F)$
			s.t.\ 
		$T {\supset} S$ 
	\item $S \in \stage(F)$, if 
		$\nexists T \in \cf(F)$ with $T_F^+ \supset S_F^+$
	\item $S \in \sem(F)$, if $S \in \adm(F)$ and 
		$\nexists T \in \adm(F)$
			s.t.\ 
		$T_F^+ \supset S_F^+$
\end{itemize}
}
\fi

\iflong{}
\fi

We will make frequent use of the 
following concepts.

\begin{definition}
\label{def:argspairs}
Given $\Ss \subseteq 2^{\Aa}$,
$\Args_\Ss$ denotes $\bigcup_{S\in\Ss} S$ and
$\Pairs_\Ss$ denotes $\{ (a,b) \mid 
\exists S\in \Ss: \{a,b\}\subseteq S\}$.
$\Ss$ is called an \emph{extension-set} (over $\Aa$) if
$\Args_\Ss$ is finite.
\end{definition}

{ 
\iflong
As is easily observed, for all considered semantics $\sigma$,
$\sigma(F)$ is an extension-set for any \af\ $F$.
\fi
} 

\mysection{Compact Argumentation Frameworks}
\label{sec:compact_afs}


\begin{definition}
\label{def:caf}
Given a semantics $\sigma$
the set of
\emph{compact argumentation frameworks}
under $\sigma$ is defined as
$
\CAF_\sigma = \{F \in \allafs \mid 
\Args_{\sigma(F)} = A_F \}
$.
We call an \af\ $F \in \CAF_\sigma$ just $\sigma$-compact.
\end{definition}


\noindent

Of course the contents of $\CAF_\sigma$
differ with respect to the semantics $\sigma$.
Concerning relations between the classes of compact \afs\
note that if for two semantics $\sigma$ and $\theta$ it holds that
$\sigma(F) \subseteq \theta(F)$ for any \af\ $F$, 
then also $\CAF_\sigma \subseteq \CAF_\theta$.
Our first important result provides a full picture of
the relations between classes of compact \afs\ under the semantics we consider.

\begin{proposition}
\label{prop:sup_naive_stage_stb}
\begin{enumerate}
\item
$\CAF_\semi \subset \CAF_\pref$;
\item
$\CAF_\stb \subset \CAF_\sigma \subset \CAF_\naive$
for $\sigma \in \{\pref,\semi,\stage\}$;
\item
$\CAF_\theta \not\subseteq \CAF_\stage$ and $\CAF_\stage \not\subseteq \CAF_\theta$
for $\theta \in \{\pref,\semi\}$.
\end{enumerate}
\end{proposition}
\begin{proof}
\noindent (1)
$\CAF_\semi \subseteq \CAF_\pref$ is by the fact
that, in any {\sc af} $F$, $\semi(F) \subseteq \pref(F)$.
Properness follows from the \af\ $F'$ in Figure~\ref{fig:caf_relations} (including the dotted part)\footnote{
The construct in the lower part of the figure
represents symmetric attacks between each pair of arguments.}.
Here $\pref(F') = \{\{z\},$
$\{x_1,a_1\},$
$\{x_2,a_2\},$
$\{x_3,a_3\},$
$\{y_1,b_1\},$
$\{y_2,b_2\},$
$\{y_3,b_3\}\}$, but
$\semi(F') = (\pref(F') \sm \{\{z\}\})$,
hence $F' \in \CAF_\prf$, but $F' \notin \CAF_\semi$.

\begin{figure}[t]
\centering
%
%
\begin{tikzpicture}[scale=0.9,>=stealth]
	\tikzstyle{arg}=[draw, thick, circle, fill=gray!15,inner sep=2pt]
		\path (0,0)     node[arg](a3){$a_3$}
			++(1.2,-0.2)	node[arg](a1){$a_1$}
			++(1.2,0.2)	node[arg](a2){$a_2$}
			++(2,0)	node[arg](b3){$b_3$}
			++(1.2,-0.2)	node[arg](b1){$b_1$}
			++(1.2,0.2)	node[arg](b2){$b_2$}
			;
		\path (0,-1.4)  node[arg](x1){$x_1$}
			++(1.2,0) node[arg](x2){$x_2$}
			++(1.2,0) node[arg](x3){$x_3$}
			++(2,0) node[arg](y1){$y_1$}
            ++(1.2,0) node[arg](y2){$y_2$}
            ++(1.2,0) node[arg](y3){$y_3$}
			;
                \path (3.4,-0.4)node[arg,dotted,inner sep=3pt](z){$z$};
		\path [->, thick]
		    (x1) edge (a3)
			(x2) edge (a1)
			(x3) edge (a2)
			(y1) edge (b3)
			(y2) edge (b1)
			(y3) edge (b2)
            (a3) edge (a1)
            (a1) edge (a2)
            (b3) edge (b1)
            (b1) edge (b2)
            [bend right, out=-20, in=-160](a2) edge (a3)
            (b2) edge (b3)
			 ;
                \draw[<->,rounded corners=4pt, thick]
				(x1) -- (0,-2) -- (6.8,-2) -- (y3);
                \draw[->,thick] (1.2,-2) -- (x2);
                \draw[->,thick] (2.4,-2) -- (x3);
                \draw[->,thick] (4.4,-2) -- (y1);
                \draw[->,thick] (5.6,-2) -- (y2);

                \path[<->,thick,dotted,out=-160,in=30] (z) edge (x1);
                \path[<->,thick,dotted,out=-135,in=30] (z) edge (x2);
                \path[<->,thick,dotted,out=-110,in=20] (z) edge (x3);
                \path[<->,thick,dotted,out=-70,in=160] (z) edge (y1);
                \path[<->,thick,dotted,out=-45,in=150] (z) edge (y2);
                \path[<->,thick,dotted,out=-20,in=150] (z) edge (y3);
                
\end{tikzpicture}
\vspace{-5pt}
\caption{\afs\ illustrating the relations between various semantics.}
\label{fig:caf_relations}
\vspace{-0.1cm}
\end{figure}
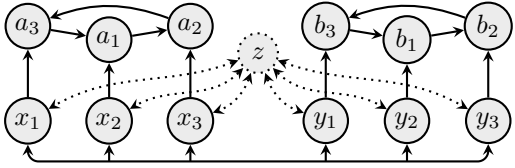

\noindent (2)
%
Let $\sigma \in \{\pref,\semi,\stage\}$.
The $\subseteq$-relations follow from the fact
that, in any {\sc af} $F$,
$\stb(F) \subseteq \sigma(F)$
and each $\sigma$-extension is, by being conflict-free,
part of some naive extension.
The {\sc af} $(\{a,b\},\{(a,b)\})$, which is compact under naive but not under $\sigma$,
and \af\ $F$ from Figure~\ref{fig:caf_relations} (now without the dotted part),
which is compact under $\sigma$ but not under stable,
show that the relations are proper.

\noindent (3)
The fact that $F'$ from Figure~\ref{fig:caf_relations}
(again including the dotted part)
is also not $\stage$-compact
shows $\CAF_\pref \not\subseteq \CAF_\stage$.
{ 
\iflong
Likewise, the \af\ $G$ 
depicted below
is $\sem$-compact, but not $\stage$-compact.

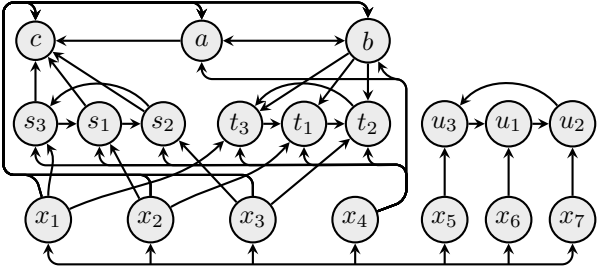
\begin{center}
\vspace{-5pt}
\begin{tikzpicture}[scale=0.85,>=stealth]
	\tikzstyle{arg}=[draw, thick, circle, fill=gray!15,inner sep=2pt]
		\path (0,0)    node[arg,inner sep=3pt](c){$c$}
                      ++(2.6,0)node[arg,inner sep=3pt](a){$a$}
                      ++(2.6,0)  node[arg,inner sep=3pt](b){$b$}
                      ;
		\path (0,-1.3)    node[arg](s1){$s_3$}
		       ++(1,0)  node[arg](s2){$s_1$}
                       ++(1,0)  node[arg](s3){$s_2$}
                       ++(1.2,0)node[arg](t1){$t_3$}
		       ++(1,0)  node[arg](t2){$t_1$}
                       ++(1,0)  node[arg](t3){$t_2$}
                       ++(1.2,0)node[arg](u1){$u_3$}
		       ++(1,0)  node[arg](u2){$u_1$}
                       ++(1,0)  node[arg](u3){$u_2$}
                       ;
                 \path (0.2,-2.8) node[arg](x1){$x_1$}
                       ++(1.6,0)node[arg](x2){$x_2$}
                       ++(1.6,0)node[arg](x3){$x_3$}
                       ++(1.6,0)node[arg](x4){$x_4$}
                       ++(1.4,0)node[arg](x5){$x_5$}
                       ++(1,0)  node[arg](x6){$x_6$}
                       ++(1,0)  node[arg](x7){$x_7$}
                       ;
              \path[->,thick,in=-60,out=90] (x1) edge (s1);
              \path[->,thick,bend right,out=0,in=-160] (x2) edge (t2)
                                                       (x1) edge (t1);
   
              \path[->,thick]
                     (x2) edge (s2)
                     (x3) edge (s3)
                     (x3) edge (t3)

                     (x5) edge (u1)
                     (x6) edge (u2)
                     (x7) edge (u3)

                     
                     (s1) edge (c)
                     (s2) edge (c)
                     (s3) edge (c)
                     (b) edge (t1)
                     (b) edge (t2)
                     (b) edge (t3)

                     (a) edge (c)

                     (s1) edge (s2)
                     (s2) edge (s3)
                     (t1) edge (t2)
                     (t2) edge (t3)
                     (u1) edge (u2)
                     (u2) edge (u3)
                     [bend right,out=-50,in=-130]
                     (s3) edge (s1)
                     (t3) edge (t1)
                     (u3) edge (u1)
                     ;
                     
                 \path[<->,thick]
                     (a) edge (b)
                     ;

                \draw[<->,rounded corners=4pt, thick]
				(x1) -- (0.2,-3.5) -- (8.4,-3.5) -- (x7);
                \draw[->,thick] (1.8,-3.5) -- (x2);
                \draw[->,thick] (3.4,-3.5) -- (x3);
                \draw[->,thick] (5.0,-3.5) -- (x4);
                \draw[->,thick] (6.4,-3.5) -- (x5);
                \draw[->,thick] (7.4,-3.5) -- (x6);
                
                \foreach \brg / \bx in {x1/0.0, x2/1.8, x3/3.4}
                \foreach \arg / \x in {c/0, a/2.6, b/5.2}
                \draw[->,rounded corners=4pt, thick]
				(\brg) -- (\bx,-2.1) -- (-0.5,-2.1) -- (-0.5,0.6) -- (\x,0.6) -- (\arg);
                
                \foreach \arg / \x in {a/2.6, b/5.5}
                \draw[->,rounded corners=4pt, thick]
                                (x4) -- (5.8,-2.5) -- (5.8,-0.6) -- (\x,-0.6) -- (\arg);
                \foreach \arg / \x in {s1/0, s2/1, s3/2, t1/3.2, t2/4.2, t3/5.2}
                \draw[->,rounded corners=4pt, thick]
                                (x4) -- (5.8,-2.5) -- (5.8,-1.94) -- (\x,-1.95) -- (\arg); 
\end{tikzpicture}
\vspace{-5pt}
\end{center}

The reason for this is that argument $a$ does not occur in any stage extension.
Although $\{a,u_1,x_5\},$ $\{a,u_2,x_6\},$ $\{a,u_3,x_7\} \in \semi(G)$,
the range of any conflict-free set containing $a$ is a proper subset of
the range of every stable extension of $G$.
$\stage(G)= \{\{c,u_i,x_4\} \mid i \in \{1,2,3\}\} \cup 
            \{\{b,u_i,s_j,x_{i+4}\} \mid i,j \in \{1,2,3\}\} \cup
            \{\{t_i,u_j,s_i,x_i\} \mid i,j \in \{1,2,3\}\}$.
Hence $\CAF_\semi \not\subseteq \CAF_\stage$.

\noindent Finally, the \af\ $(\{a,b,c\}, \{(a,b),(b,c),(c,a)\})$ 
shows $\CAF_\stage \not\subseteq \CAF_\theta$ for $\theta \in \{\pref,\semi\}$.
\fi
} 
\end{proof}

 %
%

Considering compact \afs\ obviously has effects on the computational complexity of reasoning.
While credulous and skeptical acceptance are now easy (as discussed in the introduction) the next theorem shows that verifying extensions
is still as hard as in general \afs.

\begin{theorem}\label{thm:complexity}
For $\sigma \in \{\pref, \semi, \stage\}$, {\sc af} $F=(A,R)\in \CAF_\sigma$ and $E \subseteq A$, 
it is $\coNP$-complete to decide whether $E \in \sigma(F)$.
\lvspace{-2mm}
\end{theorem}
\begin{short-or-nmr-proof}
For all three semantics the problem is known to be in $\coNP$~\cite{CaminadaCD12,DimopoulosT96,dvorwolt11}.
For hardness we only give a (prototypical) proof for $\pref$.
%
We use a standard reduction from CNF formulas $\varphi(X) = \bigwedge_{c\in C} c$ with each clause $c\in C$ 
a disjunction of literals from $X$ to an {\sc af} $F_{\varphi}$ 
with arguments $A_{\varphi} = \{\varphi, \bar{\varphi}_1,\bar{\varphi}_2,\bar{\varphi}_3\} \cup C \cup X \cup \bar{X}$ and
attacks 
(i) $\{ (c, \varphi) \mid c\!\in\!C \}$, 
(ii) $\{ (x , \bar{x}) , (\bar{x},x)\!\mid x\!\in\!X\}$,
(iii) $\{ (x ,c) \mid x \mbox{ occurs in } c \} \cup \{ (\bar{x}, c) \mid \neg x \mbox{ occurs in } c\}$,
(iv) $\{(\varphi,\bar{\varphi}_1),(\bar{\varphi}_1, \bar{\varphi}_2),(\bar{\varphi}_2, \bar{\varphi}_3),(\bar{\varphi}_3, \bar{\varphi}_1)\}$, and
(v) $\{(\bar{\varphi}_1,x),(\bar{\varphi}_1,\bar{x}) \mid x \in X\}$.
It holds that $\varphi$ is satisfiable iff 
there is an $S \neq \emptyset$ in $\sigma_1(F_\varphi)$~\cite{DimopoulosT96}.
We extend $F_{\varphi}$ with four new arguments $\{t_1,t_2,t_3,t_4\}$ and the following attacks:
(a) $\{(t_i,t_j), (t_j,t_i) \mid 1 \leq i<j \leq 4\}$,
(b) $\{(t_1,c) \mid c \in C\}$,
(c) $\{(t_2,c), (t_2, \bar{\varphi}_2) \mid c \in C\}$ and
(d) $\{(t_3,\bar{\varphi}_3)\}$.
This extended \af\ is in $\CAF_\pref$ and moreover 
$\{t_4\}$ is a preferred extension thereof iff 
$\pref(F_{\varphi})=\{\emptyset\}$ iff
$\varphi$ is unsatisfiable. 
\end{short-or-nmr-proof}
\begin{long-and-not-nmr-proof}
For all three semantics the problem is known to be in $\coNP$~\cite{CaminadaCD12,DimopoulosT96,dvorwolt11}.
The main idea of our proofs is to adapt the $\coNP$-hardness proofs for general \afs\ such that the 
constructed \afs\ fall into the class $\CAF_\sigma$.

We start with the hardness proof for $\pref$ semantics.
We use a standard reduction from CNF formulas $\varphi(X) = \bigwedge_{c\in C} c$ with each clause $c\in C$ 
a disjunction of literals from $X$ to $F_{\varphi}=(A_{\varphi},R_{\varphi})$ 
given as 
\begin{align*}
\vspace{-0.2cm}
 A_{\varphi} &= \{\varphi, \bar{\varphi}_1,\bar{\varphi}_2,\bar{\varphi}_3\} \cup C \cup X \cup \bar{X}\\
 R_{\varphi} &= \{ (c, \varphi) \mid c\in C \} \cup \{ (x , \bar{x}) , (\bar{x},x) \mid x \in X\} \cup \\
 &\quad\ \{ (x ,c) \mid x \mbox{ occurs in } c \} \cup \{ (\bar{x}, c) \mid \neg x \mbox{ occurs in } c\}\cup \\
 &\quad\ \{(\varphi,\bar{\varphi}_1),(\bar{\varphi}_1, \bar{\varphi}_2),(\bar{\varphi}_2, \bar{\varphi}_3),(\bar{\varphi}_3, \bar{\varphi}_1)\} \cup \\ 
 &\quad\ \{(\bar{\varphi}_1,x),(\bar{\varphi}_1,\bar{x}) \mid x \in X\}.
\vspace{-0.3cm}
\end{align*}
We illustrate the framework $F_{\varphi}$ (ignore the dotted part for the moment) for the CNF-formula 
$\varphi=$ $(x_1 \vee x_2 \vee x_3) \wedge (x_2 \vee \neg x_3 \vee \neg x_4) \wedge (x_2 \vee x_3 \vee x_4)$.

\begin{tikzpicture}[xscale=0.97,yscale=0.9,>=stealth]
 		\tikzstyle{arg}=[draw, thick, circle, fill=gray!15,inner sep=2pt]
		\path 	(0,1) node[arg, dotted](t2){$t_2$}
			++(-2,1) node[arg, dotted](t1){$t_1$}
			++(4,0) node[arg, dotted](t3){$t_3$}
			++(-2,1) node[arg, dotted](t4){$t_4$};
		\path [left,<->, thick, dotted]
			(t1) edge (t2)
			(t1) edge (t3)
			(t1) edge (t4)
			(t2) edge (t3)
			(t2) edge (t4)
			(t3) edge (t4)
			;
			
		\path 	node[arg](phi){$\varphi$}
			++(-3,-1) node[arg](c1){$c_1$}
			++(3,0) node[arg](c2){$c_2$}
			++(3,0) node[arg](c3){$c_3$};
		\path 	(-3.8,-2.3)  node[arg](z1){$x_1$}
			++(1,0) node[arg](nz1){$\bar x_1$}
			++(1.2,0) node[arg](z2){$x_2$}
			++(1,0) node[arg](nz2){$\bar x_2$}
			++(1.2,0) node[arg](z3){$x_3$}
			++(1,0) node[arg](nz3){$\bar x_3$}
			++(1.2,0) node[arg](z4){$x_4$}
			++(1,0) node[arg](nz4){$\bar x_4$};
		\path  (2.5,0) node[arg](nphi1){$\bar \varphi_1$}
		       ++(0.5,1) node[arg](nphi2){$\bar \varphi_2$}
		       ++(-1,0) node[arg](nphi3){$\bar \varphi_3$}
			;
		\path [left,->, thick]
			(nphi1) edge (nphi2)
			(nphi2) edge (nphi3)
			(nphi3) edge (nphi1);
		\path [left,->, thick]
			(c1) edge (phi)
			(c2) edge (phi)
			(c3) edge (phi);
		\path [left,->, thick]
			(z1) edge (c1)
			(z2) edge (c1)
			(z2) edge (c2)
			(nz3) edge (c2)
			(z3) edge (c3)
			(z4) edge (c3)
			(phi)edge (nphi1)
			[out=152,in=-10](z3) edge (c1)
			[out=28, in=-175](z2) edge (c3)
			[out=152,in=-10](nz4) edge (c2);
		\path [left,<->, thick]
			(z1) edge (nz1)
			(z2) edge (nz2)
			(z3) edge (nz3)
			(z4) edge (nz4);			
		  \foreach \x / \y in {z1/-3.8, nz1/-2.8, z2/-1.6, nz2/-0.6, z3/0.6, nz3/1.6, z4/2.8, nz4/3.8}
			\draw[left,->,rounded corners=5pt,thick]
				(nphi1) -- (4.3,0) -- (4.3,-3.05) -- (\y,-3.05) -- (\x);
		\path [left,->, thick, dotted]
			(t3) edge (nphi3)
			(t2) edge (c1)
			(t2) edge (c3)
			(t1) edge (c1)
			(t1) edge (c2)
			[bend right](t1) edge (c3)
			[bend left](t2) edge (c2)
			[bend right](t2) edge (nphi2)
			;
\end{tikzpicture}
It holds that $\varphi$ is satisfiable iff 
$\pref(F_{\varphi})=\{\emptyset\}$~\cite{DimopoulosT96}.

Now we extend $F_{\varphi}$ with four new arguments $\{t_1,t_2,t_3,t_4\}$ and the following attacks:
(i) $\{(t_i,t_j), (t_j,t_i) \mid 1 \leq i<j \leq 4\}$,
(ii) $\{(t_1,c) \mid c \in C\}$,
(iii) $\{(t_2,c), (t_2, \bar{\varphi}_2) \mid c \in C\}$ and
(iv) $\{(t_3,\bar{\varphi}_3)\}$.

It is easy to see that each preferred extension picks exactly one of the $t_i$.
(1) If we pick $t_1$ we get a preferred extension $\{t_1,\varphi, \bar{\varphi}_2\} \cup M \cup \overline{X \setminus M}$,
for each $M \subseteq X$.
(2) If we pick $t_2$ we get a preferred extension $\{t_2,\varphi, \bar{\varphi}_3\} \cup M \cup \overline{X \setminus M}$,
for each $M \subseteq X$.
(3) If we pick $t_3$ we get the preferred extension $\{t_3,\bar{\varphi}_1\} \cup C$.
As $\{t_4\}$ is admissible the constructed \af\ is $\CAF_\pref$.
Now consider (4) picking the argument $t_4$. When removing the three arguments $t_1, t_2, t_3$ attacked by $t_4$
we obtain the original framework from~\cite{DimopoulosT96}. 
So $\{t_4\}$ is a preferred extension iff this framework has the empty preferred extension iff 
$\varphi$ is unsatisfiable. 

For the hardness of $\semi$ again consider $F_{\varphi}$ and add arguments $t_1,t_2,t_3,t_4, t_5, t_6$ and associated 
attacks in a similar fashion as for $\pref$ semantics. 
That is we add attacks
(i) $\{(t_i,t_j), (t_j,t_i) \mid 1 \leq i<j \leq 6\}$,
(ii) $\{(t_1,c) \mid c \in C\}$,
(iii) $\{(t_2,c), (t_2, \bar{\varphi}_2) \mid c \in C\}$ and
(iv) $\{(t_3,\bar{\varphi}_3)\}$.
But additionally we also use arguments $g_1, g_2, g_3, h_1, h_2, h_3$  forming two odd length cycles, i.e.\ having attacks
(iv) $(g_1,g_2),(g_2,g_3),(g_3,g_1),(h_1,h_2),(h_2,h_3),(h_3,h_1)$, 
and attacks from the $t_i$ to the $g_j,h_j$, i.e.\ 
(v) $(t_1,h_1),(t_2,h_2),(t_3,h_3),(t_4,g_1),(t_5,g_2),(t_6,g_3)$.
Below we illustrate the conflicts between the new arguments.
\begin{center}
 \begin{tikzpicture}[xscale=0.97,yscale=0.9,>=stealth]
 		\tikzstyle{arg}=[draw, thick, circle, fill=gray!15,inner sep=2pt]
		\path 	(0,0) node[arg, dotted](t1){$t_1$}
			++(1,-1) node[arg, dotted](t2){$t_2$}
			++(1.5,0) node[arg, dotted](t3){$t_3$}
			++(1,1) node[arg, dotted](t4){$t_4$}
			++(-1,1) node[arg, dotted](t5){$t_5$}
			++(-1.5,0) node[arg, dotted](t6){$t_6$}
			;
		\path 	(5,0) node[arg, dotted](g1){$g_1$}
			++(0,1) node[arg, dotted](g2){$g_2$}
			++(-1,0) node[arg, dotted](g3){$g_3$}
			;
		\path 	(-1.5,0) node[arg, dotted](h1){$h_1$}
			++(0,-1) node[arg, dotted](h2){$h_2$}
			++(1,0) node[arg, dotted](h3){$h_3$}
			;
		\path [left,<->, thick, dotted]
			(t1) edge (t2)
			(t1) edge (t3)
			(t1) edge (t4)
			(t1) edge (t5)
			(t1) edge (t6)
			(t2) edge (t3)
			(t2) edge (t4)
			(t2) edge (t5)
			(t2) edge (t6)
			(t3) edge (t4)
			(t3) edge (t5)
			(t3) edge (t6)
			(t4) edge (t5)
			(t4) edge (t6)
			(t5) edge (t6)
			[->]
			(g1) edge (g2)
			(g2) edge (g3)
			(g3) edge (g1)
			(h1) edge (h2)
			(h2) edge (h3)
			(h3) edge (h1)
			(t4) edge (g1)
			(t1) edge (h1)
			[bend right]
			(t5) edge (g2)
			(t2) edge (h2)
			[bend left]
			(t6) edge (g3)
			(t3) edge (h3)
			;			
\end{tikzpicture}
\end{center}
Now, because of the mutual conflicts between the $t_i$,
a semi-stable extension can either have all three of $g_1, g_2, g_3$ or all three of $h_1,h_2,h_3$ in its range (and none of the other).
In the former case one of $t_4, t_5 ,t_6$ must be part of the extension in the latter case
one of $t_1, t_2 ,t_3$ must be in the extension.
First we show that the constructed \af\ is $\CAF_\semi$.
To this end consider the following extensions:
(1) $\{t_1, h_2, \varphi, \bar{\varphi}_2\} \cup M \cup \overline{X \setminus M}$ for each $M \subseteq X$,
(2) $\{t_2, h_3, \varphi, \bar{\varphi}_3\} \cup M \cup \overline{X \setminus M}$ for each $M \subseteq X$, and
(3) $\{t_3, h_1, \bar{\varphi}_1\} \cup C$.
It's easy to verify that in (1)-(3) only $\{g_1,g_2,g_3\}$ are missing in the range and thus the sets are semi-stable.
As the sets $\{t_4,g_2\},\{t_5,g_3\},\{t_6,g_1\}$ are isomorphic and the only ones that have $g_1, g_2, g_3$ in their range, 
they are clearly part of a semi-stable extensions. 
Hence, the constructed \af\ is $\CAF_\semi$.
Now, consider the set $\{t_4,g_2\}$. We want to know whether this set is semi-stable or not. 
Because of the observation that it has to be part of a semi-stable extension this is equivalent to the question whether
$\{t_4,g_2\}$ is a preferred extension. Again we can simple remove the arguments $\{t_4,g_2\}^+$ from the \af\ and ask whether 
the remaining \af, which is $F_{\varphi}$, has a non-empty admissible set.
Again, by \cite{DimopoulosT96}, this is equivalent to $\varphi$ being satisfiable.

To prove hardness for $\stage$ we adapt a reduction from~\cite{dvorwolt11}.
We assume that a 3-CNF formula is given as a set $C$ of clauses, where
each clause is a set over atoms and negated atoms (denoted by $\bar x$) over variables $X$.
The \af\ $F'_{\varphi}=(A,R)$ is given by
\begin{align*}
A=& X \cup \bar X \cup C \cup \{s,\varphi\}  \cup \{a,b,d\} \cup \{a_i,b_i,d_i,g_i \mid 1\leq i \leq 3\}\\
R=& \{(x,\bar x),(\bar x,x)\mid x \in X\} \cup \{(l,c)\mid l \in c, c \in C\}
 \cup\\
& \{(c,\varphi)\mid c \in C\} \cup
\{(s,x),(s, \bar{x}), (s,\varphi), (\varphi,s) \mid x \in X\} \cup \\
&\{(a,b),(b,a),(a,d),(d,a),(b,d),(d,b)\} \cup\\
&\{ (\varphi,g_i),(a,a_i),(b,b_i),(d,d_i) \mid  1\leq i \leq 3\} \cup \\
&\{(d,s),(d,c)\mid c\in C\} \cup \{(a,\varphi)\}
\end{align*}
The {\sc af} $F'_{\varphi}$ is illustrated in Figure~\ref{fig:hardnessstage}.
\begin{figure}
\begin{tikzpicture}[>=stealth, scale=0.8]
		\path 	node[arg](phi){$\varphi$}
			++(-3,-1.3) node[arg](c1){$c_1$}
			++(3,0) node[arg](c2){$c_2$}
			++(3,0) node[arg](c3){$c_3$};
		\path 	(-5,-3.2)  node[arg](z1){$y_1$}
			++(1.3,0) node[arg](nz1){$\bar x_1$}
			++(1.5,0) node[arg](z2){$x_2$}
			++(1.3,0) node[arg](nz2){$\bar x_2$}
			++(1.5,0) node[arg](z3){$x_3$}
			++(1.3,0) node[arg](nz3){$\bar x_3$}
			++(1.5,0) node[arg](z4){$x_4$}
			++(1.3,0) node[arg](nz4){$\bar x_4$};
		\path [left,->, thick]
			(c1) edge (phi)
			(c2) edge (phi)
			(c3) edge (phi);
		\path [left,->, thick]
			(z1) edge (c1)
			(z2) edge (c1)
			(z3) edge (c1)
			(nz2) edge (c2)
			(nz3) edge (c2)
			(nz4) edge (c2)
			(nz1) edge (c3)
			(nz2) edge (c3)
			(z4) edge (c3);
		\path [left,<->, thick]
			(z1) edge (nz1)
			(z2) edge (nz2)
			(z3) edge (nz3)
			(z4) edge (nz4);

		\path (2,0) node[arg,inner sep=1pt](g1){$g_1$}
			++(1,0) node[arg,inner sep=1pt](g2){$g_2$}
			++(1,0) node[arg,inner sep=1pt](g3){$g_3$};
		\draw[left,->,thick](phi) edge (g1);
		\draw[bend left,->,thick](phi) edge (g2);
		\draw[bend left,->,thick](phi) edge (g3);
		\draw (0,-5) node[arg](s){$s$};
		\path [left,->, thick]
			(s) edge (nz1)
			(s) edge (nz2)
			(s) edge (nz3)
			(s) edge (nz4)
			(s) edge (z1)
			(s) edge (z2)
			(s) edge (z3)
			(s) edge (z4)
			;
		\draw[left,<->,rounded corners=3pt]  (s) -- (-5.7,-5) -- (-5.7,0) -- (phi);
		\path (-3,-6.4) node[arg](a){$a$}
			++(3,0) node[arg](b){$b$}
			++(3,0) node[arg](d){$d$};
		\path (-4.5,-7.9) node[arg,inner sep=1pt](a1){$a_1$}
			++(1,0) node[arg,inner sep=1pt](a2){$a_2$}
			++(1,0) node[arg,inner sep=1pt](a3){$a_3$}
			++(1.5,0) node[arg,inner sep=1pt](b1){$b_1$}
			++(1,0) node[arg,inner sep=1pt](b2){$b_2$}
			++(1,0) node[arg,inner sep=1pt](b3){$b_3$}
			++(1.5,0) node[arg,inner sep=1pt](d1){$d_1$}
			++(1,0) node[arg,inner sep=1pt](d2){$d_2$}
			++(1,0) node[arg,inner sep=1pt](d3){$d_3$}
		;
		\path [left,<->, thick]
			(a) edge (b)
			(b) edge (d)
			[bend right, out=-20,in=-160](a) edge (d)
			;
		\path [left,->, thick]
			(a) edge (a1)
			(a) edge (a2)
			(a) edge (a3)
			(b) edge (b1)
			(b) edge (b2)
			(b) edge (b3)
			(d) edge (d1)
			(d) edge (d2)
			(d) edge (d3)
			(a1)edge (a2)
			(a2)edge (a3)
			(b1)edge (b2)
			(b2)edge (b3)
			(d1)edge (d2)
			(d2)edge (d3)
			(g1)edge (g2)
			(g2)edge (g3)
			[bend left]
			(a3)edge(a1)
			(b3)edge(b1)
			(d3)edge(d1)
			(g3)edge(g1)
			;
		\path [left,->, thick]
			(a) edge (phi)
			(d) edge (s)
			(d) edge (c2)
			(d) edge (c3)
			[bend left, out=-10,in=-168](d) edge (c1)
			;
\end{tikzpicture}
\caption{An illustration of the {\sc af} $F'_{\varphi}$ from the proof of Theorem~\ref{thm:complexity} for 
$\varphi=(x_1 \vee x_2 \vee x_3) \wedge (x_2 \vee \neg x_3 \vee \neg x_4) \wedge (x_2 \vee x_3 \vee x_4)$.}
\label{fig:hardnessstage}
\end{figure}
Now we have that (i) $F'_{\varphi} \in \CAF_\stage$, and 
(ii) $\{b,s,a_1,d_1,g_1\} \cup C$ is a stage extension iff $\varphi$ is unsatisfiable.

It is easy to see that each stage extension of $F'_{\varphi}$ contains exactly one of the arguments $a,b,d$.
Towards a contradiction consider a stage extension $E$ with $a,b,d \notin E$. Then there is a $b_i$ such that $b_i \notin E^+$.
Thus, $(E \setminus \{b_1,b_2,b_3\})\cup \{b\}$ has a larger range than $E$ and as it is also conflict-free we obtain the desired contradiction to $E$ being stage. Hence each stage extension must contain one of the arguments $a,b,d$.
Moreover only extensions with $a$ ($b$, or $d$ respectively) contain $\{a_1,a_2,a_,3\}$ 
($\{b_1,b_2,b_,3\}$, or $\{d_1,d_2,d_,3\}$ respectively) in their range. Thus when it comes to the maximality of the range
sets with $a$ ($b$, or $d$ respectively)  only compete with other sets containing $a$ ($b$, or $d$ respectively). 
We next show (i) and thus consider the following sets:
$\{d, \varphi, a_i, b_j\} \cup M \cup \overline{X \setminus M}$ for $M \subseteq X$, $1\leq i,j,\leq 3$; and
$\{a, s, b_i, d_j, g_k\} \cup C$ for $1\leq i,j,k \leq 3$.
All these sets can be easily verified to be stage extensions and thus $F'_{\varphi} \in \CAF_\stage$.

To show (ii) let us first assume that $\varphi$ is satisfiable and let $M$ be a model of $\varphi$.
Then the set  $\{b,a_1,d_1, \varphi\} \cup M \cup \overline{X\setminus M}$ is conflict-free and has range
$\{a,b,d,s,a_1,a_2,b_1,b_2,b_3,d_1,d_2,g_1,g_2,g_3,\varphi\} \cup X \cup \overline{X} \cup C$ which is a super set of 
$(\{b,s,a_1,d_1,g_1\}  \cup C)^+= \{a,b,d,s,a_1,a_2,b_1,b_2,b_3,d_1,d_2,g_1,g_2,\varphi\} \cup X \cup \overline{X} \cup C$. 
Hence, $\{b,s,a_1,d_1,g_1\} \cup C$ is not a stage extension.

For the other direction let us assume that $\{b,s,a_1,d_1,g_1\} \cup C$ is not a stage extension. 
Then there must be a stage extension $E$ with $E^+ \supset (\{b,s,a_1,d_1,g_1\} \cup C)^+$.
Let us consider $A_{F'} \setminus (\{b,s,a_1,d_1,g_1\} \cup C)^+ = \{g_3,a_3,d_3\}$. 
So $E$ has to have one of these $3$ arguments in its range. 
As $b_1,b_2,b_3 \in E^+$ (and thus $b\in E$) and $a_1, a_2, d_1, d_2 \in E^+$ we know that 
we cannot get $a_3 \in E^+$ or $d_3 \in E^+$.
Hence we know that $g_3 \in E^+$ and as also $g_1,g_2 \in E^+$ it must be that $\varphi \in E$.
Moreover, $C \subset E^+$ but as $E$ is conflict-free $E \cap C = \emptyset$.
If we consider the set $M= E \cap X$ we have that for each clause there is either 
a positive literal $x$ such that $x \in X$ or a negative literal $\neg x$ such that $x \not\in x$.
Hence, $M$ is a model of $\varphi$. 
\end{long-and-not-nmr-proof}

\mysection{Limits of Compact AFs}
\label{sec:numbers}

Extension-sets obtained from compact {\sc af}s satisfy certain structural properties.
Knowing these properties can help us decide whether -- given an extension-set $\Ss$ -- there is a compact {\sc af} $F$ such that $\Ss$ is exactly the set of extensions of $F$ for a semantics $\sigma$.
This is also known as \emph{realizability}:
A set $\Ss \subseteq 2^{\Aa}$ is called \emph{compactly realizable} under semantics $\sigma$ iff there is a compact {\sc af} $F$ with $\sigma(F)=\Ss$.

Among the most basic properties that are necessary for compact realizability, we find numerical aspects like possible numbers of $\sigma$-extensions. 
\iflong{%
\begin{mylongexample}
  \label{exm:realize-8,17}
  Consider the following {\sc af} $F_2$:
  
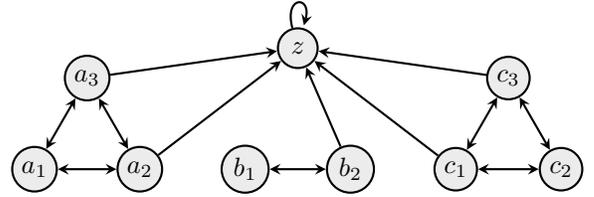
\begin{center}
    \begin{tikzpicture}[scale=0.7,>=stealth]
      \tikzstyle{arg}=[draw, thick, circle, fill=gray!15,inner sep=2pt]
      \node[arg] (a) at (0,0) {$a_1$};
      \node[arg] (b) at (2,0) {$a_2$};
      \node[arg] (c) at (1,1.732) {$a_3$};
      \path[<->,thick] (a) edge (b) (b) edge (c) (c) edge (a);
      \node[arg] (d) at (8,0) {$c_1$};
      \node[arg] (e) at (10,0) {$c_2$};
      \node[arg] (f) at (9,1.732) {$c_3$};
      \path[<->,thick] (d) edge (e) (e) edge (f) (f) edge (d);
      \node[arg] (g) at (4,0) {$b_1$};
      \node[arg] (h) at (6,0) {$b_2$};
      \path[<->,thick] (g) edge (h);

      \node[arg,,inner sep=3pt] (z) at (5,2.3) {$z$};
      \path[->,thick] (z) edge[loop above] (z);
      \path[->,thick] 
      (b) edge (z)
      (c) edge (z)
      (d) edge (z)
      (f) edge (z)
      (h) edge (z);
    \end{tikzpicture}
  \end{center}
  Let us determine the stable extensions of $F_2$.
  Clearly, taking one $a_i$, one $b_i$ and one $c_i$ yields a conflict-free set that is also stable as long as it attacks $z$.
  Thus from the $3\cdot 2\cdot 3 = 18$ combinations, only one (the set $\{a_1,b_1,c_2\}$) is not stable, whence $F_2$ has $18-1=17$ stable extensions.
  We note that this {\sc af} is not compact since $z$ occurs in none of the extensions.
  Is there an equivalent stable-compact {\sc af}?
  The results of this section will provide us with a negative answer.
\end{mylongexample}
}\fi

In \cite{BaumannS13} it was shown that there is a correspondence between the maximal number of stable extensions in argumentation frameworks and the maximal number of maximal independent
sets in undirected graphs~\cite{moonmoser}. Recently, the result was generalized to further semantics \cite{DunneDLW14} and is stated below.\footnote{In this section, unless stated otherwise we use $\sigma$ as a placeholder for stable, semi-stable, preferred, stage and naive semantics.}
For any natural number $n$ we define: 
\begin{mygather}
  \sigma_{\max}(n) = \max\left\{\left|\sigma(F)\right| \mid F\in \text{\sc af}_n \right\}
\end{mygather}
\noindent $\sigma_{\max}(n)$ returns the maximal number of $\sigma$-extensions among all {\sc AF}s with $n$ arguments.
Surprisingly, there is a closed expression for $\sigma_{\max}$.
\begin{theorem} 
  \label{the:max}
  The function $\sigma_{\max}(n) : \N \to \N$ is given by
  \begin{mygather}
    \sigma_{\max}(n) = 
    \begin{cases}
      1,  & \text{if } n=0 \text{ or } n=1,\\
      3^s, & \text{if } n\geq 2 \text{ and } n = 3s, \\
      4\cdot 3^{s-1}, & \text{if } n\geq 2 \text{ and } n = 3s + 1, \\ 
      2\cdot 3^{s}, & \text{if } n\geq 2 \text{ and } n = 3s + 2.
    \end{cases}
  \end{mygather}
\end{theorem}

What about the maximal number of $\sigma$-extensions on connected graphs? Does this number coincide with $\sigma_{\max}(n)$? The next theorem provides a negative answer to this question and thus, gives space for impossibility results as we will see.  
For a natural number $n$ define
\begin{mygather}
  \sigmamaxcon(n) = \max\left\{\left|\sigma(F)\right| \mid F\in \text{\sc af}_n, F \text{ connected}   \right\}
\end{mygather}
\noindent $\sigmamaxcon(n)$ returns the maximal number of $\sigma$-extensions among all \emph{connected} {\sc AF}s with $n$ arguments.
Again, a closed expression exists.
\begin{theorem} 
  \label{the:maxcon}
  The function $\sigmamaxcon(n) : \N \to \N$ is given by
  \begin{mygather}
    \sigmamaxcon(n) = 
    \begin{cases}
      n,  & \text{if } n\leq 5,\\
      2\cdot 3^{s-1} + 2^{s-1}, & \text{if } n\geq 6 \text{ and } n = 3s, \\
      3^{s} + 2^{s-1}, & \text{if } n\geq 6 \text{ and } n = 3s + 1, \\ 
      4\cdot 3^{s-1} + 3\cdot 2^{s-2}, & \text{if } n\geq 6 \text{ and } n = 3s + 2.
    \end{cases}
  \end{mygather}
\end{theorem}
\begin{longproof} 
  First some notations: 
  for an {\sc af} $F=(A,R)$, denote its irreflexive version by
  \mbox{$\irr(F)=(A,R\setminus\{ (a,a) \mid a\in A \})$};
  denote its symmetric version by
  \mbox{$\sym(F)=(A,R\cup\{ (b,a) \mid (a,b)\in R \}$}.
  Now for the proof. 
  ($\leq$) 
  Assume given a connected {\sc \af} $F$.
  Obviously, $\naive(F)\subseteq \naive(\sym(\irr(F)))$. Thus, $\card{\naive(F)} \leq \card{\naive(\sym(\irr(F))}$. Note that for any symmetric and irreflexive $F$, $\naive(F) = \mis(\und(F))$. Consequently, $\card{\naive(F)}\leq\card{\mis(\und(\sym(\irr(F))))}$. Fortunately, due to Theorem 2 in \cite{GriggsGG88} the maximal number of maximal independent sets in connected $n$-graphs are exactly given by the claimed value range of $\sigmamaxcon(n)$. 
  ($\geq$) Stable-realizing {\sc \af}s can be derived by the extremal graphs w.r.t. MIS in connected graphs (consider Fig. 1 in \cite{GriggsGG88}). Replacing undirected edges by symmetric directed attacks accounts for this.\\
In consideration of $\stb\subseteq\stage\subseteq\naive$ we obtain: $\sigmamaxcon(n)$ provides a tight upper bound for $\sigma\in\{\stb,\stage,\naive\}$. Finally, using $\stb\subseteq\sem\subseteq\prf$, $\prf(F)\subseteq \prf(\sym(\irr(F)))$ and $\prf(\sym(\irr(F))) = \stb(\sym(\irr(F)))$ (compare Corollary 1 in \cite{agree}) we obtain that $\sigmamaxcon(n)$ even serves for $\sigma\in\{\sem,\prf\}$. 
\end{longproof}

A further interesting question concerning arbitrary {\sc af}s is whether all natural numbers less than $\sigma_{\max}(n)$ are compactly realizable.\footnote{We sometimes speak about realizing a natural number $n$ and mean realizing an extension-set with $n$ extensions.}
The following theorem shows that there is a serious gap between the maximal and second largest number.
For any positive natural $n$ define 
\begin{mygather}
  \sigmamax{2}(n) = \max\left(\left\{\left|\sigma(F)\right| \mid F\in \text{\sc af}_n   \right\}\setminus \left\{\sigma_{\max}(n)\right\}\right)
\end{mygather}
\noindent $\sigmamax{2}(n)$ returns the second largest number of $\sigma$-extensions among all {\sc AF}s with $n$ arguments.
Graph theory provides us with an expression.

\begin{theorem} 
  \label{the:max2}
  Function $\sigmamax{2}(n) : \N\setminus\{0\} \to \N$ is given by
  \begin{mygather}
    \sigmamax{2}(n) = 
    \begin{cases}
      \sigma_{\max}(n) - 1,  & \text{if } 1 \leq n \leq 7,\\
       \sigma_{\max}(n)\cdot \frac{11}{12}, & \text{if } n\geq 8 \text{ and } n = 3s + 1, \\ 
     \sigma_{\max}(n)\cdot \frac{8}{9}, & \text{otherwise. } 
    \end{cases}
  \end{mygather}
\end{theorem} 
\begin{longproof} ($\geq$) $\sigma$-realizing {\sc \af}s can be derived by the extremal graphs w.r.t. the second largest number of MIS (consider Theorem 2.4 in \cite{secondMIS}). Replacing undirected edges by symmetric directed attacks accounts for this. This means, the second largest number of $\sigma$-extensions is at least as large as the claimed value range.\\
($\leq$) If $n \leq 7$, there is nothing to prove. Given $F\in \text{\sc af}_n$ s.t. $n\geq 8$. Suppose, towards a contradiction, that $\sigmamax{2}(n) < \card{\sigma(F)} < \sigma_{\max}(n)$. It is easy to see that for any symmetric and irreflexive $F$, $\sigma(F) = \mis(\und(F))$. Furthermore, due to Theorem 2.4 in \cite{secondMIS} the second largest numbers of maximal independent sets in $n$-graphs are exactly given by the claimed value range of $\sigmamax{2}(n)$. Consequently, $F$ cannot be symmetric and self-loop-free simultaneously. Hence, $\card{\sigma(F)} < \card{\sigma(\sym(\irr(F)))} = \sigma_{\max}(n)$. Note that up to isomorphisms the extremal graphs are uniquely determined (cf. Theorem 1 in \cite{GriggsGG88}). Depending on the remainder of $n$ on division by 3 we have $K_3$'s for $n\equiv 0$, either one $K_4$ or two $K_2$'s and the rest are $K_3$'s in case of $n\equiv 1$ and one $K_2$ plus $K_3$'s for $n\equiv 2$. Consequently, depending on the remainder we may thus estimate \mbox{$\card{\sigma(F)} \leq k\cdot\sigma_{\max}(n)$} where $k\in\{\frac{2}{3},\frac{3}{4},\frac{1}{2}\}$. Since ($\geq$) is already shown we finally state $l\cdot\sigma_{\max}(n)\leq\sigmamax{2}(n) < \card{\sigma(F)} \leq \frac{3}{4}\cdot\sigma_{\max}(n)$ where $l\in\{\frac{11}{12},\frac{8}{9}\}$. This is a clear contradiction concluding the proof.    
\end{longproof}

To showcase the intended usage of these theorems, we now prove that the {\sc af} $F_2$ seen earlier indeed has no equivalent compact {\sc af}.

\begin{example} 
  \label{ex:nonrealmax} 
  Recall that the (non-compact) {\sc af} $F_2$ we discussed previously had the extension-set
  $\Ss$ with $|\Ss| = 17$ and $|\Args_\Ss| = 8$. 
  Is there a stable-compact {\sc af} with the same extensions?
  Firstly, nothing definitive can be said by Theorem~\ref{the:max} since $17\leq 18 = \sigma_{\max}(8)$. Furthermore, in accordance with Theorem~\ref{the:maxcon} the set $\Ss$ cannot be compactly $\sigma$-realized by a connected {\sc af} since $17 > 15 = \sigmamaxcon(8)$. Finally, using Theorem~\ref{the:max2} we infer that
	the set $\Ss$ is not compactly $\sigma$-realizable because $\sigmamax{2}(8) = 16 < 17 < 18 = \sigmamax{}(8)$.
\end{example}

The compactness property is instrumental here, since Theorem~\ref{the:max2} has no counterpart in non-compact {\sc af}s.
More generally, allowing additional arguments as long as they do not occur in extensions enables us to realize any number of stable extensions up to the maximal one.

\begin{proposition}
  \label{p:non-compact-realizable}
  Let $n$ be a natural number.
  For each $k\leq \sigmamax{}(n)$, there is an {\sc af} $F$ with 
  $|\Args_{\stb(F)}|=n$ and $|\stb(F)|=k$.
\end{proposition}
\begin{longproof}
  To realize $k$ stable extensions with $n$ arguments, we start with the construction for the maximal number from Theorem~\ref{the:max}.
  We then subtract extensions as follows:
  We choose $\sigmamax{}(n)-k$ arbitrary distinct stable extensions of the {\sc af} realizing the maximal number.
  To exclude them, we use the construction of Def.~9 in \cite{DunneDLW14}.
\end{longproof}

\iflong{\ifnmr\else{
  \begin{longcorollary}
    \label{c:non-compact-realizable}
    Let $n$ be a natural number and $\sigma$ among preferred, semi-stable and stage semantics.
    For each $k\leq \sigmamax{}(n)$, there is an {\sc af} $F$ with 
    $|\Args_{\sigma(F)}|=n$ and $\sigma(F)=k$.
  \begin{proof}
    Follows from Lemmata~2.2 and 4.2 in \cite{DunneDLW14}.
  \end{proof}
  \end{longcorollary}
}\fi}\fi

%

Now we are prepared to provide possible short cuts when deciding realizability of a given extension-set
by initially simply counting the extensions. 
First some formal definitions.

\begin{definition}
Given an \af\ $F=(A,R)$,
the component-structure
$\KC(F) = \{K_1,\dots,K_n\}$
of $F$
is the set of sets of arguments,
where each $K_i$ coincides with
the arguments of a weakly connected component of the underlying graph;
\mbox{$\KCt(F) = \{K \in \KC(F) \mid |K| \geq 2\}$}.
\end{definition}

\iflong{
\begin{mylongexample}
The \af\ $F=(\{a,b,c\},\{(a,b)\})$ has component-structure 
$\KC(F) =\{\{a,b\},\{c\}\}$.
\end{mylongexample}
}\fi

The component-structure $\KC(F)$
gives information about the number $n$ of components of $F$
as well as the size $|K_i|$ of each component.
Knowing the components of an \af,
computing the $\sigma$-extensions
can be reduced to
computing the $\sigma$-extensions of each component
and building the cross-product.
The \af\ resulting from restricting $F$ to component $K_i$ is given by
$\red{F}{K_i}=(K_i,R_F\cap K_i\times K_i)$.

\begin{lemma}
\label{lemma:cross-product}
Given an \af\ $F$ with component-structure $\KC(F) = \{K_1,\dots,K_n\}$
it holds that the
extensions in $\sigma(F)$ and
the tuples in $\sigma(\red{F}{K_1}) \times \dots \times \sigma(\red{F}{K_n})$
are in one-to-one correspondence.
\end{lemma}
\begin{long-and-not-nmr-proof}
  By induction on $n$; the base case $n=1$ is trivial.
  For the induction step let $\KC(F) = \{K_1,\dots,K_n,K_{n+1}\}$.

  ``$\subseteq$'':
  Let $S\in\sigma(F)$.
  Define $D_{n+1}=S\cap K_{n+1}$.
  By induction hypothesis, there are sets $D_1,\ldots,D_n$ such that
  each $D_i$ is a $\sigma$-extension of $\red{F}{K_i}$ and
  $S\setminus K_{n+1} = D_1\cup\dots\cup D_n$.
  We have to show that $D_{n+1}$ is a $\sigma$-extension of $\red{F}{K_{n+1}}$.
  $\sigma=\stb$:
  Clearly $D_{n+1}$ is conflict-free, and any $a\in K_{n+1}\setminus D_{n+1}$ is attacked since $S$ is stable and the attacks must come from $D_{n+1}$ due to connectivity.
  $\sigma\in\{\naive,\pref\}$:
  If there is a conflict-free/admissible superset of $D_{n+1}$, then $S$ is not naive/preferred for $F$.
  $\sigma\in\{\stage,\sem\}$:
  If there is a superset of $D_{n+1}$ with greater range, then $S$ is not stage/semi-stable for $F$.

  ``$\supseteq$'':
  Let $D_1,\ldots,D_n,D_{n+1}$ such that
  each $D_i$ is a $\sigma$-extension of $\red{F}{K_i}$.
  Define $S = D_1\cup\dots\cup D_n\cup D_{n+1}$; we show that $S\in\sigma(F)$.
  By induction hypothesis, $D_1\cup\dots\cup D_n \in \sigma(\red{F}{K_1,\dots,K_n})$.
  $\sigma=\stb$:
  Clearly $S$ is conflict-free since all $D_i$ are conflict-free;
  since $D_{n+1}$ is stable for $\red{F}{K_{n+1}}$ it attacks all $a\in K_{n+1}\setminus D_{n+1}$ and thus $S$ is stable for $F$.
  $\sigma\in\{\naive,\prf\}$:
  If $S$ is not naive/preferred for $F$, there is a conflict-free/admissible superset of $S$ in $F$.
  There is at least one additional argument, that is either in $D_1\cup\dots\cup D_n$ or in $D_{n+1}$.
  But the first is impossible due to induction hypothesis, and the second due to presumption.
  $\sigma\in\{\stage,\sem\}$:
  If $S$ is not stage/semi-stable for $F$, there is a conflict-free/admissible set $S'$ with greater range.
  The range difference must manifest itself in $D_1\cup\dots\cup D_n$ or $D_{n+1}$, which leads to a contradiction with the induction hypothesis and the presumption that $D_{n+1}$ is stage/semi-stable for $\red{F}{K_{n+1}}$.
\end{long-and-not-nmr-proof}

Given an extension-set $\Ss$
we want to decide whether $\Ss$ is realizable
by a compact \af\ under semantics $\sigma$.
For an \af\ $F=(A,R)$ with $\sigma(F)=\Ss$
we know that there cannot be a conflict between any pair of arguments in $\Pairs_\Ss$,
hence $R \subseteq \noPairs{\Ss} = (A \times A) \setminus \Pairs_\Ss$.
In the next section, we will show that it is highly non-trivial to decide which of the attacks in $\noPairs{\Ss}$ can be and should be used to realize $\Ss$.
For now, the next proposition implicitly shows that
for argument-pairs $(a,b) \notin \Pairs_\Ss$,
although there is not necessarily a direct conflict between $a$ and $b$,
they are definitely in the same component.

\begin{proposition}
\label{prop:comp-structure}
\label{prop:component-structure}
Given an extension-set $\Ss$,
the component-structure $\KC(F)$
of any \af\ $F$ compactly realizing $\Ss$ under semantics $\sigma$ 
($F \in \CAF_\sigma$, $\sigma(F) = \Ss$)
is 
given by the equivalence classes of the transitive closure of $\noPairs{\Ss}$, $\noPairsCl{\Ss}$.
\end{proposition}
\begin{longproof}
Consider some extension-set $\Ss$ together with
an \af\ $F \in \CAF_\sigma$ with $\sigma(F) = \Ss$.
We have to show that for any pair of arguments $a,b \in \Args_\Ss$
it holds that
$(a,b) \in \noPairsCl{\Ss}$ iff
$a$ and $b$ are connected in the graph underlying $F$.

If $a$ and $b$ are connected in $F$,
this means that there is a sequence $s_1,\dots,s_n$
such that $a=s_1$, $b=s_n$, and $(s_1,s_2),\dots,(s_{n-1},s_n) \notin \Pairs_\Ss$,
hence $(a,b) \in \noPairsCl{\Ss}$.

If $(a,b) \in \noPairsCl{\Ss}$ then also
there is a sequence $s_1,\dots,s_n$
such that $a=s_1$, $b=s_n$, and $(s_1,s_2),\dots,(s_{n-1},s_n) \in \noPairs{\Ss}$.
Consider some $(s_i,s_{i+1}) \in \noPairs{\Ss}$
and assume, towards a contradiction, that
$s_i$ occurs in another component of $F$ than $s_{i+1}$.
Recall that $F \in \CAF_\sigma$, so each of $s_i$ and $s_{i+1}$ occur in some extension and $\sigma(F) \neq \emptyset$.
Hence, by Lemma~\ref{lemma:cross-product},
there is some $\sigma$-extension $E \supseteq \{s_i,s_{i+1}\}$ of $F$,
meaning that $(s_i,s_{i+1}) \in \Pairs_\Ss$,
a contradiction.
Hence all $s_i$ and $s_{i+1}$ for $1 \leq i < n$ occur in the same component of $F$,
proving that also $a$ and $b$ do so.
\end{longproof}

We will denote the component-structure induced by an extension-set $\Ss$ as $\KC(\Ss)$.
Note that, by Proposition~\ref{prop:comp-structure}, $\KC(\Ss)$ is equivalent
to $\KC(F)$ for every $F \in \CAF_\sigma$ with $\sigma(F) = \Ss$.
Given $\Ss$, the computation of $\KC(\Ss)$ can be done in polynomial time.
With this 
we can use results from graph theory
together with number-theoretical considerations
in order to get impossibility results for compact realizability.

\iflong{
Recall that for a single connected component with $n$ arguments the maximal number of stable extensions is denoted by $\sigmamaxcon(n)$ and its values are given by Theorem~\ref{the:maxcon}.
In the compact setting it further holds for a connected \af\ $F$ with at least $2$ arguments that $\sigma(F) \geq 2$.
}\fi

\begin{proposition}
\label{prop:not_odd}
Given an extension-set $\Ss$ where $|\Ss|$ is odd,
it holds that
if $\exists K \in \KC(\Ss) : |K| = 2$
then $\Ss$ is not compactly realizable under semantics $\sigma$.
\end{proposition}
\begin{longproof}
Assume to the contrary that there is an $F \in \CAF_\sigma$ with $\sigma(F)=\Ss$.
We know that $\KC(F) = \KC(\Ss)$.
By assumption there is a $K\in\KC(\Ss)$ with $|K|=2$, whence $|\sigma(K)|=2$.
Thus by Lemma~\ref{lemma:cross-product} the total number of $\sigma$-extensions is even.
Contradiction.
\end{longproof}

\begin{example}
\label{ex:7_exts}
Consider the extension-set
$\Ss = \{\{a,b,c\},$ $\{a,b',c'\},$ $\{a',b,c'\},$
$\{a',b',c\},$ $\{a,b,c'\},$ $\{a',b,c\},$ $\{a,b',c\}\} = \stb(F_1)$ 
where $F_1$ is the non-compact {\sc af} from the introduction.
There, it took us some effort to argue that $\Ss$ is not compactly $\stb$-realizable.
Proposition~\ref{prop:not_odd} now gives an easier justification:
$\Pairs_\Ss$ yields $\KC(\Ss)  = \{\{a,a'\},\{b,b'\},\{c,c'\}\}$.
Thus $\Ss$ with $|\Ss|=7$ cannot be realized.
\end{example}

We \iflong{ denote the set of possible numbers of $\sigma$-extensions
of a compact \af\ with $n$ arguments as $\PC(n)$;
likewise we }\fi denote the set of possible numbers of $\sigma$-extensions
of a compact and \emph{connected} \af\ with $n$ arguments as $\PC^c(n)$.
\iflong{
Although we know that $p \in \PC(n)$ implies $p \leq \sigmamax{}(n)$,
there may be $q \leq \sigmamax{}(n)$
which are not realizable by a compact \af\ under $\sigma$; likewise for $q\in \PC^c(n)$.
}\fi
\iflong{

Clearly, any $p\leq n$ is possible by building an undirected graph with $p$ arguments where every argument attacks all other arguments, a $K_p$, and filling up with $k$ isolated arguments ($k$ distinct copies of $K_1$) such that $k+p=n$.
This construction obviously breaks down if we want to realize more extensions than we have arguments, that is, $p > n$.
In this case, we have to use Lemma~\ref{lemma:cross-product} and further graph-theoretic gadgets for addition and even a limited form of subtraction.
Space does not permit us to go into too much detail, but let us show how for $n=7$ any number of extensions up to the maximal number $12$ is realizable.
\newcommand{\sarg}[1]{\draw[fill=black] #1 circle (3pt)}
\def\scale{0.3}
For $12=3\cdot 4$, Theorem~\ref{the:max} yields the realization, a disjoint union of a $K_3$ and a $K_4$ 
(\begin{tikzpicture}[scale=\scale]
  \sarg{(0,0)};
  \sarg{(1,0)};
  \sarg{(0.5,0.866)};
  \draw (0,0)--(0.5,0.866)--(1,0)--(0,0);
  \sarg{(2,0)};
  \sarg{(2,1)};
  \sarg{(3,0)};
  \sarg{(3,1)};
  \draw (2,0) -- (3,0) -- (3,1) -- (2,1) -- (2,0) -- (3,1) -- (3,0) -- (2,1);
\end{tikzpicture}).
For the remaining numbers, we have that \mbox{$8=2\cdot 4\cdot 1$} and so we can combine a $K_2$, a $K_4$ and a $K_1$ (
  \begin{tikzpicture}[scale=\scale]
    \sarg{(0,0)};
    \sarg{(1,0)};
    \draw (0,0)--(1,0);
    \sarg{(2,0)};
    \sarg{(2,1)};
    \sarg{(3,0)};
    \sarg{(3,1)};
    \draw (2,0) -- (3,0) -- (3,1) -- (2,1) -- (2,0) -- (3,1) -- (3,0) -- (2,1);
    \sarg{(4,0)};
  \end{tikzpicture}).
Likewise,
$9=3\cdot 3\cdot 1$ 
(\begin{tikzpicture}[scale=\scale]
  \sarg{(0,0)};
  \sarg{(1,0)};
  \sarg{(0.5,0.866)};
  \draw (0,0)--(0.5,0.866)--(1,0)--(0,0);
  \sarg{(2,0)};
  \sarg{(3,0)};
  \sarg{(2.5,0.866)};
  \draw (2,0)--(2.5,0.866)--(3,0)--(2,0);
  \sarg{(4,0)};
\end{tikzpicture});
$10=3\cdot 3 + 1$
(\begin{tikzpicture}[scale=\scale]
  \sarg{(0,0)};
  \sarg{(1,0)};
  \sarg{(0.5,0.866)};
  \draw (0,0)--(0.5,0.866)--(1,0)--(0,0);
  \sarg{(2,0)};
  \sarg{(3,0)};
  \sarg{(2.5,0.866)};
  \draw (2,0)--(2.5,0.866)--(3,0)--(2,0);
  \sarg{(1.5,1)};
  \draw (0,0) -- (1.5,1) -- (1,0) -- (1.5,1) -- (0.5,0.866) -- (1.5,1) -- (2,0) -- (1.5,1) -- (3,0) -- (1.5,1) -- (2.5,0.866);
\end{tikzpicture}) and finally
$11=3\cdot 4-1$
(\begin{tikzpicture}[scale=\scale]
  \sarg{(0,0)};
  \sarg{(1,0)};
  \sarg{(0.5,0.866)};
  \draw (0,0)--(0.5,0.866)--(1,0)--(0,0);
  \sarg{(2,0)};
  \sarg{(2,1)};
  \sarg{(3,0)};
  \sarg{(3,1)};
  \draw (2,0) -- (3,0) -- (3,1) -- (2,1) -- (2,0) -- (3,1) -- (3,0) -- (2,1);
  \draw (1,0) -- (2,0);
\end{tikzpicture}).
These small examples already show that $\PC$ and $\PC^c$ are closely intertwined and let us deduce some general corollaries:
Firstly, $\PC^c(n)\subseteq \PC(n)$ since connected \afs\ are a subclass of \afs.
Next, $\PC(n)\subseteq\PC(n+1)$ as in the step from
\begin{tikzpicture}[scale=\scale]
  \sarg{(0,0)};
  \sarg{(1,0)};
  \sarg{(0.5,0.866)};
  \draw (0,0)--(0.5,0.866)--(1,0)--(0,0);
  \sarg{(2,0)};
  \sarg{(3,0)};
  \sarg{(2.5,0.866)};
  \draw (2,0)--(2.5,0.866)--(3,0)--(2,0);
\end{tikzpicture}
to
\begin{tikzpicture}[scale=\scale]
  \sarg{(0,0)};
  \sarg{(1,0)};
  \sarg{(0.5,0.866)};
  \draw (0,0)--(0.5,0.866)--(1,0)--(0,0);
  \sarg{(2,0)};
  \sarg{(3,0)};
  \sarg{(2.5,0.866)};
  \draw (2,0)--(2.5,0.866)--(3,0)--(2,0);
  \sarg{(4,0)};
\end{tikzpicture}.
We even know that $\PC(n)\subsetneq\PC(n+1)$ since $\sigmamax{}(n+1)\in\PC(n+1)\setminus\PC(n)$.
Furthermore, whenever $p\in\PC(n)$, then $p+1\in\PC^c(n+1)$, as in the step from 
\begin{tikzpicture}[scale=\scale]
  \sarg{(0,0)};
  \sarg{(1,0)};
  \sarg{(0.5,0.866)};
  \draw (0,0)--(0.5,0.866)--(1,0)--(0,0);
  \sarg{(2,0)};
  \sarg{(3,0)};
  \sarg{(2.5,0.866)};
  \draw (2,0)--(2.5,0.866)--(3,0)--(2,0);
\end{tikzpicture}
to
\begin{tikzpicture}[scale=\scale]
  \sarg{(0,0)};
  \sarg{(1,0)};
  \sarg{(0.5,0.866)};
  \draw (0,0)--(0.5,0.866)--(1,0)--(0,0);
  \sarg{(2,0)};
  \sarg{(3,0)};
  \sarg{(2.5,0.866)};
  \draw (2,0)--(2.5,0.866)--(3,0)--(2,0);
  \sarg{(1.5,1)};
  \draw (0,0) -- (1.5,1) -- (1,0) -- (1.5,1) -- (0.5,0.866) -- (1.5,1) -- (2,0) -- (1.5,1) -- (3,0) -- (1.5,1) -- (2.5,0.866);
\end{tikzpicture}.
The construction that goes from $12$ to $11$ above obviously only works if there are two weakly connected components overall, which underlines the importance of the component structure of the realizing \af.
}\fi
\iflong{
Indeed, multiplication of extension numbers of single components is our only chance to achieve overall numbers that are substantially larger than the number of arguments.
This is what we will turn to next.
}\fi
Having to leave the exact contents of $\PC(n)$ and $\PC^c(n)$ open, 
we can still state the following result:

\begin{proposition}
\label{prop:product_of_components}
Let $\Ss$ be an extension-set that is compactly realizable under semantics $\sigma$ where $\KCt(\Ss)=\{ K_1, \ldots, K_n \}$.
Then for each $1\leq i\leq n$ there is a $p_i\in \PC^c(|K_i|)$ such that $|\Ss| = \prod_{i=1}^n p_i$.
\end{proposition}
\begin{longproof}
First note that components of size $1$
can be ignored since they have no impact on the number of $\sigma$-extensions.
Lemma~\ref{lemma:cross-product} also implies that
the number of $\sigma$-extensions of an \af\ with multiple components
is the product of the number of $\sigma$-extensions of each component.
Since the factor of any component $K_i$ must be in $\PC^c(|K_i|)$ the result follows.
\end{longproof}

\begin{example}
\label{ex:4_exts}
Consider the extension-set
$\Ss' = \{\{a,b,c\},\{a,b',c'\},$ $\{a',b,c'\},$ $\{a',b',c\}\}$.
(In fact there exists a (non-compact) \af\ $F$ with $\stb(F)=\Ss'$).
We have the same component-structure $\KC(\Ss') = \KC(\Ss)$ as in Example~\ref{ex:7_exts},
but since now $|\Ss'|=4$ we cannot use Proposition~\ref{prop:not_odd}
to show impossibility of realization in terms of a compact \af.
But with Proposition~\ref{prop:product_of_components}
at hand we can argue in the following way:
$\PC^c(2) = \{2\}$ and since $\forall K \in \KC(\Ss'): |K|=2$
it must hold that $|\Ss| = 2 \cdot 2 \cdot 2 = 8$,
which is obviously not the case.
\end{example}

In particular, we have a straightforward non-realizability criterion whenever $|\Ss|$ is prime:
the \af\ (if any) must have at most one weakly connected component of size greater than two.
Theorem~\ref{the:maxcon} gives us the maximal number of $\sigma$-extensions in a single weakly connected component.
Thus whenever the number of desired extensions is larger than that number and prime, it cannot be realized.
\begin{corollary}
  Let extension-set $\Ss$ with $|\Args_\Ss|=n$ be compactly realizable under $\sigma$.
  If $|\Ss|$ is a prime number, then $|\Ss|\leq\sigmamaxcon(n)$.
\end{corollary}

\iflong{
\begin{mylongexample}
  Let $\Ss$ be an extension-set with $|\Args_\Ss|=9$ and $|\Ss|=23$.
  We find that 
  \mbox{$\sigmamaxcon(9)=2\cdot 3^{2}+2^2=22<23=|\Ss|$}
  and thus $\Ss$ is not compactly realizable under semantics $\sigma$.
\end{mylongexample}
}\fi

We can also make use of the derived component structure of an extension-set $\Ss$.
Since the total number of extensions of an \af\ is the product of these numbers for its weakly connected components (Lemma~\ref{lemma:cross-product}), each non-trivial component contributes a non-trivial amount to the total.
Hence if there are more components than the factorization of $|\Ss|$ has primes in it, then $\Ss$ cannot be realized.

\begin{corollary}
Let extension-set $\Ss$ be compactly realizable under $\sigma$ and 
\mbox{$f_1^{z_1} \cdot \ldots \cdot f_m^{z_m}$} be the integer factorization of $|\Ss|$, where $f_1,\dots,f_m$ are prime numbers.
Then \mbox{$z_1+\ldots + z_m \geq |\KCt(\Ss)|$}.
\end{corollary}

\iflong{
\begin{mylongexample}
Consider an extension-set $\Ss$
containing $21$ extensions
and $|\KC(\Ss)|=3$.
Since $21 = 3^1*7^1$ and further $1+1 < 3$,
$\Ss$ is not compactly realizable under semantics $\sigma$.
\end{mylongexample}
}\fi

\mysection{Capabilities of Compact AFs}
\label{sec:capabilites}

The results in the previous section made clear
that the restriction to compact \afs\ entails
certain limits in terms of compact realizability.
Here we provide some results
approaching an exact characterization
of the capabilities of compact \afs\
with a focus on stable semantics.

\iflong{\vspace*{3mm}}\fi
\mysubsection{C-Signatures}
\label{subsec:signatures}

The \emph{signature} of a semantics $\sigma$
is defined as $\Sigma_\sigma = \{\sigma(F) \mid F \in \allafs\}$
and contains all possible sets of extensions an \af\ can possess under $\sigma$ 
(see
\cite{DunneDLW14}
for 
characterizations of such signatures).
We first 
provide alternative, yet equivalent, characterizations of the signatures
of some the semantics under consideration.
Then we strengthen the concept of signatures to ``compact'' signatures (c-signatures),
which contain 
all extension-sets realizable with compact \af s.



The most central concept when structurally analyzing extension-sets
is captured by the $\Pairs$-relation from Definition~\ref{def:argspairs}.
Whenever two arguments $a$ and $b$ occur jointly in some element $S$ of extension-set $\Ss$
(i.e.\ $(a,b) \in \Pairs_\Ss$)
there cannot be a conflict between those arguments in an {\sc af}
having $\Ss$ as solution under any standard semantics.
$(a,b) \in \Pairs_\Ss$ can be read as ``evidence of no conflict'' between $a$ and $b$ in $\Ss$.
Hence, the $\Pairs$-relation gives rise to sets of arguments
that are conflict-free in any \af\ realizing $\Ss$.

\iflong{\vspace*{2mm}}\fi
\begin{definition}\label{df:x}
Given an extension-set $\Ss$, we define
\iflong
\begin{itemize}
\item $\Ss^\cf = \{S \subseteq \Args_\Ss \mid \forall a,b \in S : (a,b) \in \Pairs_\Ss\}$;
\item $\Ss^+ = \max_\subseteq\; \Ss^\cf$.
\end{itemize}
\fi
\end{definition}


\iflong{ 
To show that the characterizations of signatures
in Proposition~\ref{prop:signatures} below
are indeed equivalent to the ones given in \cite{DunneDLW14}
we first recall some definitions from there.

\begin{mylongdefinition}
For an extension-set \mbox{$\Ss \subseteq 2^\Aa$},
the \emph{downward-closure} of $\Ss$ is defined as
\mbox{$\dcl(\Ss) = \{S'\subseteq S \mid S \in \Ss\}$}.
Moreover, $\Ss$ is called
\begin{itemize}
\ifnmr{
\item
\emph{incomparable}, if for all $S,S' \in \Ss$,
$S \subseteq S'$ implies $S{=}S'$,
\item
\emph{tight} if
for all \mbox{$S \in \Ss$} and \mbox{$a \in \Args_\Ss$} it holds that
if \mbox{$(S \cup \{a\}) \notin \Ss$}
then there exists an \mbox{$s \in S$} such that \mbox{$(a,s) \notin \Pairs_\Ss$}.
}
\else{
\item
\emph{incomparable}, if for each $S,S' \in \Ss$,
$S \subseteq S'$ implies $S = S'$,
\item
\emph{tight} if
for all $S \in \Ss$ and $a \in \Args_\Ss$ it holds that
if $(S \cup \{a\}) \notin \Ss$
then there exists an $s \in S$ such that $(a,s) \notin \Pairs_\Ss$, and
\item
\emph{adm-closed} if for all $A,B \in \Ss$ the following holds:
if $(a,b) \in \Pairs_\Ss$ for each $a,b \in (A\cup B)$
then also $(A \cup B) \in \Ss$.
}
\fi
\end{itemize}
\end{mylongdefinition}

}\fi 

\begin{proposition}
\label{prop:signatures}
\iflong{
\ifnmr{
$\Sigma_{\naive}\! =\!\{\Ss\neq \emptyset \mid \Ss = \Ss^+ \};$\\
$\Sigma_{\stb}\!=\!\{ \Ss \mid \Ss \subseteq \Ss^+ \};$ 
$\Sigma_{\stage}\!=\!\{ \Ss \neq \emptyset \mid \Ss \subseteq \Ss^+ \}$.
}
\else{
\begin{align*}
 \Sigma_{\naive} &= \{\Ss\neq \emptyset \mid \Ss = \Ss^+ \} \\
 \Sigma_{\stb} &= \{ \Ss \mid \Ss \subseteq \Ss^+ \} \\
 \Sigma_{\stage} &= \{ \Ss \neq \emptyset \mid \Ss \subseteq \Ss^+ \} \\
 \Sigma_{\pref} &= \{ \Ss \neq \emptyset \mid \forall A,B \in \Ss \;(A\neq B)\; \forall C \in \Ss^+ : (A \cup B) \not\subseteq C \} \\
 \Sigma_{\semi} &= \{ \Ss \neq \emptyset \mid \forall A,B \in \Ss \;(A\neq B)\; \forall C \in \Ss^+ : (A \cup B) \not\subseteq C \}.
\end{align*}
}
\fi
}
\fi
\end{proposition}
\begin{longproof}
Being aware of Theorem~1 from \cite{DunneDLW14}
we have to show that,
given an extension-set $\Ss\subseteq 2^\Aa$ the following hold:
\begin{enumerate}
\item
$\Ss$ is incomparable and tight iff $\Ss \subseteq \Ss^+$,
\item
$\Ss$ is incomparable and $\dcl(\Ss)$ is tight iff $\Ss = \Ss^+$\ifnmr{.
}
\fi
\end{enumerate}

\noindent (1)
$\Rightarrow$:
Consider an incomparable and tight extension-set $\Ss$
and assume that $\Ss \not\subseteq \Ss^+$.
To this end let $S \in \Ss$ with $S \notin \Ss^+$.
Since $S \in \Ss^\cf$ by definition,
there must be some
$S' \supset S$ with $S' \in \Ss^+$.
$S' \notin \Ss$ holds by incomparability of $\Ss$.
But $S' \in \Ss^+$ means that there is some $a \in (S' \sm S)$
such that $\forall s \in S : (a,s) \in \Pairs_\Ss$,
a contradiction to the assumption that $\Ss$ is tight.

\noindent
$\Leftarrow$:
Let $\Ss$ be an extension-set such that $\Ss \subseteq \Ss^+$.
Incomparability is clear.
Now assume, towards a contradiction,
that are some $S\in\Ss$ and $a \in \Args_\Ss$
such that $(S \cup \{a\}) \notin \Ss$ and $\forall s \in S : (a,s) \in \Pairs_\Ss$.
Then there is some $S' \supseteq (S \cup \{a\})$ with $S' \in \Ss^+$,
a contradiction to $S \in \Ss^+$.

\noindent (2)
$\Rightarrow$:
Consider an incomparable extension-set $\Ss$
where $\dcl(\Ss)$ is tight and
assume that $\Ss \neq \Ss^+$.
Note that $\Pairs_\Ss = \Pairs_{\dcl(\Ss)}$.
Since $\dcl(\Ss)$ being tight implies that $\Ss$ is tight
(cf.\ Lemma~2.1 in \cite{DunneDLW14}),
$\Ss \subseteq \Ss^+$ follows by (1).
Now assume there is some $S \in \Ss^+$
with $S \notin \Ss$.
Note that $|S|\geq 3$.
Now let $S' \subset S$ and $a \in (S \sm S')$
such that $S' \in \dcl(\Ss)$ and $(S' \cup \{a\}) \notin \dcl(\Ss)$.
Such an $S'$ exists since
for each pair of arguments $a,b \in S'$, $(a,b) \in \Pairs_\Ss$ holds
as $S \in \Ss^+$.
Since also $\forall s \in S' : (a,s) \in \Pairs_\Ss$,
we get a contradiction to the assumption that $\dcl(\Ss)$ is tight.

\noindent
$\Leftarrow$:
Consider an extension-set $\Ss$ with $\Ss = \Ss^+$.
Incomparability is straight by definition.
Now assume, towards a contradiction,
that are some $S\in\dcl(\Ss)$ and $a \in \Args_\Ss$
such that $(S \cup \{a\}) \notin \dcl(\Ss)$ and $\forall s \in S : (a,s) \in \Pairs_\Ss$.
Then $(S \cup \{a\}) \in \Ss^\cf$, and
moreover there is some $S' \supseteq (S\cup \{a\})$ with $S' \in \Ss^+$ and $S' \notin \Ss$,
a contradiction to $\Ss = \Ss^+$.
\ifnmr{}
\fi
\end{longproof}





Let us now turn to signatures for compact \af s.

\begin{definition}
\label{def:caf,cs}
The \emph{c-signature} $\Sigma_\sigma^c$ of a semantics $\sigma$ is defined as 
\iflong
\[
 \Sigma_\sigma^c = \{\sigma(F) \mid F \in \CAF_\sigma \}.
\]
\fi
\end{definition}

It is clear that $\Sigma_\sigma^c \subseteq \Sigma_\sigma$
holds for any semantics.
The following result is mainly by the fact that the canonical \af\
%
%
%
\vspace{-2pt}
\[
F^\cf_\Ss = (A^\cf_\Ss,R^\cf_\Ss) = (\Args_\Ss, (\Args_\Ss \times \Args_\Ss) \sm \Pairs_\Ss)
\vspace{-1pt}
\]
%
%
%
%
%
%
%
has $\Ss^+$ as extensions under all semantics under consideration and
by extension-sets obtained from non-compact \afs\ which definitely cannot
be transformed to equivalent compact \afs.

\iflong{
The following technical lemma makes this clearer.

\begin{mylonglemma}
\label{lemma:realizability}
Given a non-empty extension-set $\Ss$, it holds that
$\sigma(F^\cf_\Ss) = \Ss^+$ where $\sigma \in \{\naive,\stb,\stage,\prf,\semi\}$.
\end{mylonglemma}
\begin{longproof}
%
%
$\naive$: The set $\naive(F^\cf_\Ss)$ contains the $\subseteq$-maximal
elements of $\cf(F^\cf_\Ss)$ just as $\Ss^+$ does of $\Ss^\cf$.
Therefore $\naive(F^\cf_\Ss) = \Ss^+$ follows directly from
the obvious fact that $\cf(F^\cf_\Ss) = \Ss^\cf$.

\noindent
$\stb,\stage,\prf,\sem$: Follow from the fact that
for the symmetric {\sc af} $F^\cf_\Ss$,
$\naive(F^\cf_\Ss) = \stb(F^\cf_\Ss) = \stage(F^\cf_\Ss) = \prf(F^\cf_\Ss) = \sem(F^\cf_\Ss)$
\cite{Coste-MarquisDM05}.
\end{longproof}
}\fi

\begin{proposition}
\label{prop:cf_naive_signature}
It holds that (1) $\Sigma^c_\naive = \Sigma_\naive$; and (2) $\Sigma^c_\sigma \subset \Sigma_\sigma$ for $\sigma \in \{\stb,\stage,\semi,\pref\}$.
\end{proposition}
\begin{longproof}
$\Sigma^c_\naive = \Sigma_\naive$
follows directly from the facts
that $\naive(F^\cf_\Ss) = \Ss^+$
(cf.\ Lemma~\ref{lemma:realizability})
and $F^\cf_\Ss \in \CAF_\naive$.

\noindent
$\stb,\stage$: Consider the extension-set
$
\Ss=\{
\{a,b,c\},$
$\{a,b,c'\},$
$\{a,b',c\},$
$\{a,b',c'\},$
$\{a',b,c\},$
$\{a',b,c'\},$
$\{a',b',c\}
\}  
$
from the example in the introduction.
It is easy to verify that $\Ss \subseteq \Ss^+$,
thus $\Ss\in\Sigma_\stb$ and $\Ss\in\Sigma_\stage$.
The \af\ realizing $\Ss$ under $\stb$ and $\stage$ is $F_1$ from the introduction.
We now show that there is no  {\sc af} $F=(\Args_\Ss,R)$
such that 
$\stb(F)=\Ss$ or $\stage(F)=\Ss$.
First, given that the sets in $\Ss$ must be conflict-free the only possible attacks in $R$ are 
$(a,a'),$ $(a',a),$ $(b,b'),$ $(b',b),$ $(c,c'),$ $(c',c)$. 
We next argue that all of them must be in $R$.
First consider the case of $\stb$. 
As $\{a,b,c\}\in\stb(F)$  it attacks $a'$ and the only chance to do so is $(a,a') \in R$ and
similar as $\{a',b,c\}\in\stb(F)$ it attacks $a$ and the only chance to do so is $(a',a) \in R$.
By symmetry we obtain $\{(b,b'),(b',b),(c,c'),(c',c)\} \subseteq R$.
Now let us consider the case of $\stage$.
As $\{a,b,c\}\in\stage(F)\subseteq \naive(F)$ either  $(a,a') \in R$ or  $(a',a) \in R$.
Consider  $(a,a') \not\in R$ then $\{a',b,c\}_F^+ \supset \{a,b,c\}_F^+$, contradicting that 
$\{a,b,c\}$ is a stage extension. The same holds for pairs $(b,b')$ and $(c,c')$; thus 
for both cases we obtain
$R=\{(a,a'),$ $(a',a),$ $(b,b'),$ $(b',b),$ $(c,c'),$ $(c',c)\}$.
However, for the resulting framework $F=(A,R)$,
we have that $\{a',b',c'\} \in \stb(F) = \stage(F)$, but
$\{a',b',c'\} \not\in \Ss$.
Hence we know that $\Ss \notin \Sigma^c_\stb$.

$\prf,\semi$:
Let $\sigma \in \{\pref,\semi\}$
and consider
$\Ss=
\{\{a,b\},$
$\{a,c,e\},$
$\{b,d,e\}\}$.  
\ifnmr{
The figure below shows an \af\ (with additional arguments)
realizing $\Ss$ under $\pref$ and $\semi$.
Hence $\Ss\in\Sigma_\sigma$ holds.
}
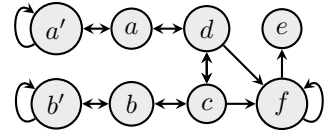
\fi

\begin{center}
\vspace{-5pt}
\begin{tikzpicture}[scale=1,>=stealth]
		\path 	node[arg](a'){$a'$}
			++(0,-1) node[arg](b'){$b'$}
			++(1,1) node[arg](a){$a$}
			++(0,-1) node[arg](b){$b$}
			++(1,0) node[arg](c){$c$}
			++(0,1) node[arg](d){$d$}
			++(1,0) node[arg](e){$e$}
			++(0,-1) node[arg](f){$f$}
			;

		\path [left,<->, thick]
			(a') edge (a)
			(b') edge (b)
			(a) edge (d)
			(b) edge (c)
			(c) edge (d)
			(d) edge (c)
			[->]
			(c) edge (f)
			(d) edge (f)
			(f) edge (e)
			[loop,thick, distance=0.4cm,out=-145, in=145,->]
			(a') edge (a')
			(b') edge (b')
			[loop,thick, distance=0.4cm,out=-35, in=35,->]
			(f) edge (f)
			;
\end{tikzpicture}
\vspace{-5pt}
\end{center}

Now  suppose there exists an \af\ $F=(\Args_\Ss,R)$
such that $\sigma(F)=\Ss$. Since $\{a,c,e\},\{b,d,e\}\in\Ss$, 
it is clear that $R$ must not contain an edge involving $e$. 
But then, $e$ is contained in each $E\in\sigma(F)$.
It follows
that $\sigma(F)\neq \Ss$.
\end{longproof}


For ordinary signatures it holds that
$\Sigma_\naive \subset \Sigma_\stage = (\Sigma_\stb \sm \{\emptyset\}) \subset
\Sigma_\sem = \Sigma_\pref$ \cite{DunneDLW14}.
This picture changes when considering the relationship of c-signatures.


\begin{proposition}
$\Sigma_\pref^c \not\subseteq \Sigma_\stb^c$;
$\Sigma_\pref^c \not\subseteq \Sigma_\stage^c$;
$\Sigma_\pref^c \not\subseteq \Sigma_\semi^c$;
$\Sigma_\naive^c \subset \Sigma_\sigma^c$ for $\sigma \in \{\stb,\stage,\semi\}$;
$\Sigma_\stb^c \subseteq \Sigma_\semi^c$;
$\Sigma_\stb^c \subseteq \Sigma_\stage^c$.
\end{proposition}
\begin{proof}
$\Sigma^c_\pref \not\subseteq \Sigma^c_\stb$,
$\Sigma_\pref^c \not\subseteq \Sigma_\stage^c$:
For the extension-set
$\Ss =$ $\{\{a,b\},$
$\{a,x_1,s_1\},$
$\{a,y_1,s_2\},$
$\{a,z_1,s_3\},$
$\{b,x_2,s_1\},$
$\{b,y_2,s_2\},$
$\{b,z_2,s_3\}\}$ it does not hold
that $\Ss \subseteq \Ss^+$
(as $\{a,b,s_1\},$ $\{a,b,s_2\},$ $\{a,b,s_3\} \in \Ss^\cf$, hence $\{a,b\} \notin \Ss^+$),
but there is a compact {\sc af} $F$ realizing $\Ss$
under the preferred semantics,
namely the one depicted in Figure~\ref{fig:prf_comp_real}.
Hence $\Sigma^c_\pref \not\subseteq \Sigma^c_\stb$
and $\Sigma^c_\pref \not\subseteq \Sigma^c_\stage$.

\begin{figure}[t]
\centering
\begin{tikzpicture}[scale=0.87,>=stealth]
	\tikzstyle{arg}=[draw, thick, circle, fill=gray!15,inner sep=2pt]
		\path (2,0)node[arg,inner sep=3pt](b){$b$}
			++(1,0)	node[arg,inner sep=3pt](a){$a$}
			;
		\path (0,-1.2)    node[arg](x1){$x_1$}
			++(1,0) node[arg](x2){$x_2$}
			++(1,0) node[arg](y1){$y_1$}
			++(1,0) node[arg](y2){$y_2$}
                        ++(1,0) node[arg](z1){$z_1$}
                        ++(1,0) node[arg](z2){$z_2$}
			;
		\path 	(0.5,-2.3) node[arg](s3){$s_3$}
			++(2,0) node[arg](s1){$s_1$}
			++(2,0) node[arg](s2){$s_2$}
			;
		\path [<->, thick] 
		        (a) edge (x2)
			(a) edge (y2)
			(a) edge (z2)
			(b) edge (x1)
			(b) edge (y1)
			(b) edge (z1)
			 ;	
		\path [->, thick]
		        (x1) edge (s3)
			(x2) edge (s3)
			(y1) edge (s1)
			(y2) edge (s1)
			(z1) edge (s2)
			(z2) edge (s2)

                        (s3) edge (s1)
                        (s1) edge (s2)
                        [bend left,out=20,in=160](s2) edge (s3)
			 ;
                \draw[<->,rounded corners=4pt, thick]
				(x1) -- (0,-1.8) -- (5,-1.8) -- (z2);
                \draw[->,thick] (1,-1.8) -- (x2);
                \draw[->,thick] (2,-1.8) -- (y1);
                \draw[->,thick] (3,-1.8) -- (y2);
                \draw[->,thick] (4,-1.8) -- (z1);
\end{tikzpicture}
\vspace{-11pt}
\caption{{\sc af} compactly realizing an extension-set $\Ss \not\subseteq \Ss^+$ under $\prf$.}
\label{fig:prf_comp_real}
\vspace{-0.1cm}
\end{figure}

{ 
\iflong
$\Sigma_\pref^c \not\subseteq \Sigma_\semi^c$:
Let $\Tt =$
$(\Ss \cup \{
\{x_1,x_2,s_1\},$
$\{y_1,y_2,s_2\},$
$\{z_1,z_2,s_3\}\})$
and assume there is some $F=(\Args_\Tt,R)$
compactly realizing $\Tt$ under the semi-stable semantics.
Consider the extensions $S = \{a,x_1,s_1\}$ and $T= \{x_1,x_2,s_1\}$.
There must be a conflict between $a$ and $x_2$,
otherwise $(S \cup T) \in \semi(F)$.
If $(a,x_2) \in R$ then,
since $T$ must defend itself and $(s_1,a),(x_1,a) \in \Pairs_\Tt$,
also $(x_2,a)\in R$.
On the other hand if $(x_2,a) \in R$ then,
since $\{a,b\}$ must defend itself and $(b,x_2) \in \Pairs_\Tt$,
also $(a,x_2) \in R$.
Hence, by all symmetric cases we get
$\{(a,\alpha_1),(\alpha_1,a),(b,\alpha_2),(\alpha_2,b)\mid \alpha\in\{x,y,z\}\} \subseteq R$.
Now as $U = \{a,b\} \in \Tt$ and $U$ must not be in conflict with any of
$s_1$, $s_2$, and $s_3$,
each $s_i$ must have an attacker which is not attacked by any $a$, $b$, or $s_i$.
Hence wlog.\ $\{(s_1,s_2),(s_2,s_3),(s_3,s_1)\} \subseteq R$.
Again consider extension $S$ and observe that $s_1$ must be defended from $s_3$,
hence $(x_1,s_3) \in R$.
We know that $S_F^+ \supseteq (\Args_\Tt \sm \{y_1,z_1\})$.
Now we observe that $S$ has to attack both $y_1$ and $z_1$ since
otherwise either $S$ would not defend itself
or $y_1$ (resp.\ $z_1$) would have to be part of $S$.
But this leads us to a contradiction because $S_F^+ = \Args_\Tt$,
but $U_F^+ \subset \Args_\Tt$, meaning that $U$ cannot be a semi-stable extension of $F$.
$\Sigma_\pref^c \not\subseteq \Sigma_\semi^c$ now follows from the fact that
$\pref(F') = \Tt$ for
$F'=(A_F,R_F \sm \{(\alpha_1,\alpha_2),(\alpha_2,\alpha_1) \mid \alpha\in\{x,y,z\}\})$
where $F$ is the \af\ depicted in Figure~\ref{fig:prf_comp_real}.
\fi
} 

$\Sigma_\naive^c \subset \Sigma_\sigma^c$ for $\sigma \in \{\stb,\stage,\semi\}$:
First of all note that any extensions-set
compactly realizable under $\naive$
is compactly realizable under $\sigma$
(by making the \af\ symmetric).
Now consider the extension-set
$\Ss = \{
\{a_1,b_2,b_3\},
\{a_2,b_1,b_3\},
\{a_3,b_1,b_2\}\}$.
$\Ss \neq \Ss^+$ since $\{b_1,b_2,b_3\} \in\Ss^+$,
hence $\Ss \notin \Sigma^c_\naive$.
{ 
\iflong
$\Sigma^c_\naive \subset \Sigma^c_\sigma$
follows from the fact that the \af\ 
below
compactly realizes $\Ss$ under $\sigma$.

\begin{center}
\vspace{-5pt}
\begin{tikzpicture}[scale=0.75,>=stealth]
	\tikzstyle{arg}=[draw, thick, circle, fill=gray!15,inner sep=2pt]
		\path 	  node[arg](a1){$a_1$}
			++(2,0) node[arg](a2){$a_2$}
			++(2,0) node[arg](a3){$a_3$}
			(0,1.1)node[arg](b1){$b_1$}
			++(2,0) node[arg](b2){$b_2$}
			++(2,0) node[arg](b3){$b_3$}
			;
		\path [<->, thick]
			(a1) edge (a2)
			(a2) edge (a3)
			[bend left,out=20,in=160](a3) edge (a1)
			;
		\path [->, thick] 
			(a1) edge (b1)
			(a2) edge (b2)
			(a3) edge (b3)
			 ;	
\end{tikzpicture}
\vspace{-5pt}
\end{center}
\fi
} 

$\Sigma_\stb^c \subseteq \Sigma_\semi^c$, $\Sigma_\stb^c \subseteq \Sigma_\stage^c$:
Follow from the fact that
$\stage(F) = \semi(F) = \stb(F)$ for any $F \in \CAF_\stb$ \cite{CaminadaCD12}.
\end{proof}



\mysubsection{The Explicit-Conflict Conjecture}
\label{sec:ecc}

So far we only have exactly characterized c-signatures
for 
the naive semantics (Proposition~\ref{prop:cf_naive_signature}).
Deciding membership of an extension-set in the c-signature
of the other semantics is more involved.
In what follows we focus on stable semantics in order to illustrate
difficulties and subtleties in this endeavor.

Although there are, as Proposition~\ref{prop:sup_naive_stage_stb} showed,
more compact {\sc af}s for $\naive$ than for $\stb$,
one can express a greater diversity of outcomes with the stable semantics,
i.e.\ $\Ss = \Ss^+$ does not necessarily hold.
Consider some \af\ $F$ with $\Ss = \stb(F)$.
By Proposition~\ref{prop:signatures} we know that
$\Ss\subseteq\Ss^+$ must hold.
Now we want to compactly realize extension-set $\Ss$ under $\stb$.
If $\Ss = \Ss^+$,
then we can obviously find a compact {\sc af}
realizing $\Ss$ under $\stb$,
since $F^\cf_\Ss$ 
will do so.
On the other hand, if $\Ss \neq \Ss^+$ we have to find
a way to handle the argument-sets in 
%
$\Ss^- = \Ss^+ \sm \Ss$.
In words, each $S \in \Ss^-$ is a $\subseteq$-maximal set with evidence of no conflict,
which is not contained in $\Ss$.




Now consider some {\sc af} $F' \in \CAF_\stb$ having $\Ss \subsetneq \Ss^+$ as its stable extensions.
Further take some $S \in \Ss^-$.
There cannot be a conflict within $S$ in $F'$,
hence we must be able to map $S$ to some argument $t \in (\Args_\Ss \sm S)$
not attacked by $S$ in $F'$.
Still, the collection of these mappings must fulfill certain conditions
in order to preserve a justification for all $S \in \Ss$ to be a stable extension
and not to give rise to other stable extensions.
We make these things more formal.

\lvspace{2mm}
\begin{definition}
Given an extension-set $\Ss$,
an \emph{exclusion-mapping} is the set
\iflong
\[
    \RX_\Ss = \bigcup_{S \in \Ss^-} \{ (s,\ff_\Ss(S)) \mid s \in S \text{ s.t. } (s,\ff_\Ss(S)) \notin \Pairs_\Ss \}
\]
\fi
where \mbox{$\ff_\Ss : \Ss^- \to \Args_\Ss$}
is a function with \mbox{$\ff_\Ss(S)  \in (\Args_\Ss \sm S)$}.\!
\end{definition}

\begin{definition}
\label{def:independent}
A set $\Ss \subseteq 2^\Aa$ is called \emph{independent}
if 
there exists an antisymmetric exclusion-mapping $\RX_\Ss$
such that it holds that
\iflong
\[
    \forall S \in \Ss \forall a \in (\Args_\Ss \sm S) :
    \exists s \in S : (s,a) \notin (\RX_\Ss \cup \Pairs_\Ss).
\]
\fi
\end{definition}

The concept of independence suggests that
the more separate the elements of some extension-set $\Ss$ are,
the less critical is $\Ss^-$.
An independent $\Ss$
allows to find the required orientation of attacks
to exclude sets from $\Ss^-$ from the stable extensions
without interferences.

\begin{theorem}\label{thm:mainstable}
For every independent extension-set $\Ss$ with $\Ss \subseteq \Ss^+$
it holds that $\Ss \in \Sigma^c_\stb$.
%
%
\lvspace{-2mm}
\end{theorem}
\begin{proof}
Consider,
given an independent extension-set $\Ss$ and
an antisymmetric ex\-clu\-sion-mapping $\RX_\Ss$ fulfilling the independence-condition
(cf.\ Definition~\ref{def:independent}),
the {\sc af}
$F_\Ss^\stb = (Args_\Ss, R^\stb_\Ss) \text{ with } R^\stb_\Ss = (R^\cf_\Ss \sm \RX_\Ss)$.
We show that $\stb(F_\Ss^\stb) = \Ss$.
First note that $\stb(F^\cf_\Ss) = \Ss^+ \supseteq \Ss$. 
As $\RX_\Ss$ is antisymmetric, one direction of each symmetric attack of $F_\Ss^\cf$ is still in $F_\Ss^\stb$.
Hence $\stb(F^\stb_\Ss) \subseteq \Ss^+$.

\noindent
$\stb(F^\stb_\Ss) \subseteq \Ss$:
Consider some $S \in \stb(F^\stb_\Ss)$
and assume that $S \notin \Ss$,
i.e.\ $S \in \Ss^-$.
Since $\RX_\Ss$ is an exclusion-mapping fulfilling the independence-condition by assumption,
there is an argument $\ff_\Ss(S) \in (\Args_\Ss \sm S)$
such that $\{(s,\ff_\Ss(S)) \mid s \in S, (s,\ff_\Ss(S)) \notin \Pairs_\Ss\} \subseteq \RX_\Ss$.
But then, by construction of $F_\Ss^\stb$,
there is no $a \in S$ such that $(a,\ff_\Ss(S)) \in R_\Ss^\stb$, 
a contradiction to $S \in \stb(F^\stb_\Ss)$.

\noindent
$\stb(F^\stb_\Ss) \supseteq \Ss$:
Consider some $S \in \Ss$
and assume that $S \notin \stb(F^\stb_\Ss)$.
We know that 
$S$ is conflict-free in $F^\stb_\Ss$.
Therefore there must be some $t \in (\Args_\Ss \sm S)$
with $S \not\at_{F^\stb_\Ss} t$.
Hence $\forall s \in S : (s,t) \in (\Pairs_\Ss \cup \RX_\Ss)$,
a contradiction to the assumption that $\Ss$ is independent.
\end{proof}


\begin{corollary}\label{cor:3}
For every 
$\Ss\in\Sigma_\stb$,
with $|\Ss| \leq 3$,
$\Ss \in \Sigma^c_\stb$.
%
%
\end{corollary}
\begin{longproof}
It is easy to see that for an extension-set $\Ss$
with $|\Ss| \leq 3$
it holds that $|\Ss^-| \leq 1$.
If $\Ss^- = \emptyset$ we are done;
if $\Ss^- = \{S\}$ observe
that by $\Ss \subseteq \Ss^+$
for each $T \in \Ss$ there is some $t \in T$ with $t \notin S$.
Hence choosing arbitrary $T \in \Ss$ and $t \in T$ with $t \notin S$
yields the antisymmetric exclusion-mapping
$\RX_\Ss = \{(s,t) \mid s \in S \text{ s.t. } (s,t) \notin \Pairs_\Ss\}$
which fulfills the independence-condition from Definition~\ref{def:independent}.
\end{longproof}


Theorem~\ref{thm:mainstable} gives a sufficient condition
for an extension-set to be contained in $\Sigma^c_\stb$.
Section~4
provided
necessary conditions with respect to number of extensions.
As these conditions do not match,
we have not arrived at an exact 
characterization of the c-signature for stable 
semantics yet.
\ifnmr{
In what follows, we identify the missing
step which we have to leave open but, as we will see, results
in an interesting problem of its own. 
}\fi
Let us first define a further class of 
frameworks.

\begin{definition}
We call an \af\ $F=(A,R)$ conflict-explicit under semantics $\sigma$
iff for each $a,b\in A$ such that $(a,b) \notin \Pairs_{\sigma(F)}$, we find 
$(a,b)\in R$ or $(b,a)\in R$ (or both).
\end{definition}


{ 
\iflong
In words, a framework is conflict-explicit under $\sigma$ 
if any two arguments of the framework which do not occur 
together in any $\sigma$-extension are explicitly conflicting, i.e.\
they are linked via the attack relation. 
\fi
} 

As a simple example consider the \af\ 
$F=(\{a,b,c,d\},$ $\{(a,b),$ $(b,a),$ $(a,c),$ $(b,d)\})$
which has
$\Ss=\stb(F)=\{\{a,d\},\{b,c\}\}$. Note that $(c,d)\notin\Pairs_\Ss$
but $(c,d)\notin R$ as well as $(d,c)\notin R$. Thus $F$
is not conflict-explicit under stable semantics. However,
if we add attacks $(c,d)$ or $(d,c)$ we obtain an equivalent
(under stable semantics) conflict-explicit (under stable 
semantics) \af.


\begin{theorem}
For each compact \af\ $F$ which is conflict-explicit under $\stb$,
it holds that
$\stb(F)$ is independent.
\end{theorem}
\begin{proof}
%
Consider some $F \in \CAF_\stb$
which is conflict-explicit under $\stb$
and let $\Ee = \stb(F)$.
Observe that $\Ee \subseteq \Ee^+$.
{ 
\iflong
We have to show that there exists an antisymmetric exclusion-mapping $\RX_\Ss$
fulfilling the independence-condition from Definition~\ref{def:independent}.
Let
\fi
} 
$\RX_\Ee = \{(b,a) \notin R \mid (a,b) {\in} R\}$ and
consider the {\sc af} $F^s = (A_F, R_F \cup \RX_\Ee)$ being the symmetric version of $F$.
Now let $E \in \Ee^-$.
Note that $E \in \cf(F) = \cf(F^s)$.
But as $E \notin \Ee$ there must be some $t \in (A \sm E)$
such that for all $e \in E$,
$(e,t) \notin R_F$.
For all such $e \in E$ with $(e,t) \notin \Pairs_\Ee$
it holds, as $F$ is conflict-explicit under $\stb$,
that $(t,e) \in R_F$, hence $(e,t) \in \RX_\Ee$,
showing that $\RX_\Ee$ is an exclusion-mapping.
%

It remains to show that $\RX_\Ee$ is antisymmetric and
$\forall E \in \Ee \forall a \in \Args_\Ss \sm E :
\exists e \in E : (e,a) \notin (\RX_\Ee \cup \Pairs_\Ee)$ holds.
As some pair $(b,a)$ is in $\RX_\Ee$ iff $(a,b) \in R$ and $(b,a) \notin R$,
$\RX_\Ee$ is antisymmetric.
Finally consider some $E \in \Ee$ and $a \in \Args_\Ss \sm E$
and assume that $\forall e \in E : (e,a) \in \RX_\Ee \Or (e,a) \in \Pairs_\Ee$.
This means that $e \not\at_F a$,
a contradiction to
$E$ being a stable extension of $F$.
\end{proof}

Since our characterizations
of signatures completely abstract away from the actual structure of 
\afs\ but only focus on the set of extensions, 
our problem would be solved if the following was true.


\medskip
\noindent
{\bf EC-Conjecture.}\;
For each \af\ $F=(A,R)$
there exists an \af\ $F'=(A,R')$ 
which is conflict-explicit under the stable semantics
such that $\stb(F) = \stb(F')$.
\medskip

{ 
\iflong
Note that the EC-conjecture implies that 
for each compact \af, there exists a stable-equivalent 
conflict-explicit (under stable) \af.
\fi
} 

\begin{theorem}
\label{thm:stb_c-signature}
Under the assumption that the EC-conjecture holds,
\iflong
\[
  \Sigma^c_\stb = \{ \Ss \mid \Ss \subseteq \Ss^+ \And \Ss \text{ is independent}\}.
\]
\fi
\end{theorem} 

Unfortunately, the question whether an equivalent conflict-explicit
\af\ exists is not as simple as the example above suggests.
We provide a
few examples showing that proving the conjecture 
includes some subtle issues. 
Our first example shows that for adding missing attacks, 
the orientation of the attack needs to be carefully chosen.

\ifnmr{
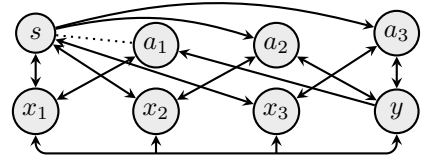
\begin{figure}[t]
\centering
\begin{tikzpicture}[scale=0.8,>=stealth]
	\tikzstyle{arg}=[draw, thick, circle, fill=gray!15,inner sep=2pt]
		\path (0,0)node[arg,inner sep=3pt](s){$s$}
			++(2,-0.2)node[arg](a1){$a_1$}
                        ++(2,0) node[arg](a2){$a_2$}
                        ++(2,0.2) node[arg](a3){$a_3$}
			;
		\path (0,-1.3)    node[arg](x1){$x_1$}
			++(2,0) node[arg](x2){$x_2$}
			++(2,0) node[arg](x3){$x_3$}
			++(2,0) node[arg,,inner sep=3pt](y){$y$}
                        ;
		\path [<->, thick] 
		        (x1) edge (a1)
			(x1) edge (s)
			(x2) edge (a2)
			(x2) edge (s)
			(x3) edge (a3)
			(x3) edge (s)
                        (y) edge (a2)
                        (y) edge (a3)
			 ;	
		\path [->, thick]
		        (y) edge (a1)
                        [bend left,out=15,in=165](s) edge (a2)
                        [bend left,out=15,in=165](s) edge (a3)
			 ;
                \draw[<->,rounded corners=4pt, thick]
				(x1) -- (0,-2.0) -- (6,-2.0) -- (y);
                \draw[->,thick] (2,-2) -- (x2);
                \draw[->,thick] (4,-2) -- (x3);
                \path[dotted,thick] (s) edge (a1); 
\end{tikzpicture}
\vspace{-6pt}
\caption{Orientation of non-explicit conflicts matters.}
\label{fig:orientation}
\vspace{-0.1cm}
\end{figure}
}
\fi

\begin{example}
\label{ex:orientation}
\ifnmr{
Consider the \af\ $F$ in Figure~\ref{fig:orientation}
and observe $\stb(F)=$
$\{\{a_1,a_2,x_3\},$
$\{a_1,a_3,x_2\},$
$\{a_2,a_3,x_1\},$
$\{s,y\}\}$.
}
\fi

\noindent
$\Pairs_{\stb(F)}$ yields one pair of arguments $a_1$ and $s$ whose conflict is not explicit by $F$,
i.e.\ $(a_1,s) \notin \Pairs_{\stb(F)}$, but $(a_1,s),(s,a_1) \notin R_F$.
Now adding the attack $a_1 \at_F s$ to $F$ would reveal the additional stable extension 
$\{a_1,a_2,a_3\} \in (\stb(F))^+$.
On the other hand by adding the attack $s \at_F a_1$ we get the conflict-explicit \af\ $F'$
with $\stb(F) = \stb(F')$.

Finally recall the role of the arguments $x_1$, $x_2$, and $x_3$.
Each of these arguments enforces exactly one extension (being itself part of it)
by attacking (and being attacked by) all arguments not in this extension.
We will make use of this construction-concept in Example~\ref{ex:more_than_orientation}.
\end{example}


Even worse, it is sometimes necessary to not only add the missing
conflicts but also change the orientation of existing attacks
such that the missing attack ``fits well''.

\begin{example}
\label{ex:more_than_orientation}
\lvspace{-1mm}
Let $X = \{
x_{s,t,i},
x_{s,u,i},
x_{t,u,i} \mid 1 \leq i \leq 3 \} \cup$
$\{ 
x_{a,1,2},
x_{a,1,3},
x_{a,2,3} \}$  and
$\Ss=
\{\{s_i,t_i,x_{s,t,i}\},$ $\{s_i,u_i,x_{s,u,i}\},$ $\{t_i,u_i,x_{t,u,i}\} \mid $ $i \in \{1,2,3\}\}
\cup$
$\{\{a_1,a_2,x_{a,1,2}\},$ $\{a_1,a_3,x_{a,1,3}\},$ $\{a_2,a_3,x_{a,2,3}\}\}$.
Consider the \af\ 
$F = (A' \cup X, R' \cup \bigcup_{x\in X}\{(x,b),(b,x) \mid b \in (A' \sm \Ss_x) \}\cup\{(x,x')\mid x,x'\in X, x\neq x'\})$,
where 
the essential part $(A',R')$
is depicted in
Figure~\ref{fig:more_than_orientation} and
$\Ss_x$ is the unique set $X\in\Ss$ with $x \in X$.
We have $\stb(F) =  \Ss$.
Observe that $F$ contains three non-explicit conflicts under the stable semantics,
namely the argument-pairs $(a_1,s_1)$, $(a_2,s_2)$, and $(a_3,s_3)$.
Adding any of $(s_i,a_i)$ to $R_F$ would turn $\{s_i,t_i,u_i\}$
into a stable extension;
adding all $(a_i,s_i)$ to $R_F$ would yield $\{a_1,a_2,a_3\}$ as additional stable extension.
Hence there is no way of making the conflicts explicit without changing other parts of $F$
and still getting a stable-equivalent \af.
Still, we can realize $\stb(F)$ by a compact and conflict-explicit \af,
for example by
$G = (A_F, (R_F \cup \{(a_1,s_1),(a_2,s_2),(a_3,s_3)\}) \sm \{(a_1,x_{a,2,3}),(a_2,x_{a,1,3}),(a_3,x_{a,1,2})\})$.
\end{example}

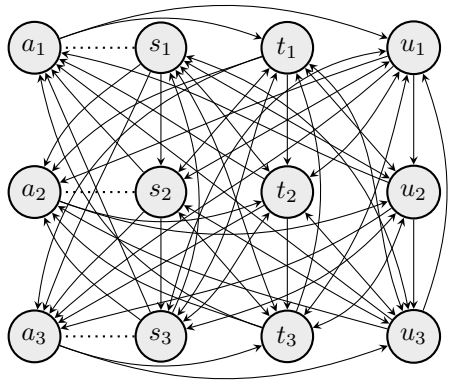
\begin{figure}[t]
\centering
\begin{tikzpicture}[scale=0.6,>=stealth]
	\tikzstyle{arg}=[draw, thick, circle, fill=gray!15,inner sep=3pt]
		\path 	(0,0)node[arg](a1){$a_1$}
                        ++(2.8,0) node[arg](s1){$s_1$}
                        ++(2.8,0) node[arg](t1){$t_1$}
                        ++(2.8,0) node[arg](u1){$u_1$}
			;
		\path (0,-3.2)    node[arg](a2){$a_2$}
			++(2.8,0) node[arg](s2){$s_2$}
			++(2.8,0) node[arg](t2){$t_2$}
			++(2.8,0) node[arg](u2){$u_2$}
                        ;
		\path (0,-6.4)    node[arg](a3){$a_3$}
			++(2.8,0) node[arg](s3){$s_3$}
			++(2.8,0) node[arg](t3){$t_3$}
			++(2.8,0) node[arg](u3){$u_3$}
                        ;
		\path [<->]

			(s1) edge (u2)
			
                        (t1) edge (s2)
                        
                        (t1) edge (u2)
                        
                        (u1) edge (s2)

			(s2) edge (t3)

                        (t2) edge (u3)
                        [bend left,out=20,in=167](u2) edge (s3)
                        
			[bend left]
                        (s1) edge (u3)
                        [bend right,out=-40]
                        [bend right, in=-155, out=-42] (u1) edge (s3)
                        [bend left, in=-180, out=30](t1) edge (u3)
                        [bend left, out=20, in=150](u2) edge (t3)
                        [in=165,out=10](u1) edge (t3)
                        [out=10](u1) edge (t2)
                        [bend right,out=0,in=165](t2) edge (s3)
                        [bend right,out=0,in=165](t1) edge (s3)
                        [bend right,out=5,in=180](s1) edge (t2)
                        [bend right,out=10,in=180](s1) edge (t3)
                        [bend left,out=2,in=-180](u3) edge (s2)
                        ;	
		\path [->]
                        (s1) edge (s2)
                        (s2) edge (s3)
                        (t1) edge (t2)
                        (t2) edge (t3)
                        (u1) edge (u2)
                        (u2) edge (u3)
                        [bend left,out=-26,in=-160](s3) edge (s1)
                        [bend right,out=-20,in=-160](t3) edge (t1)
                        [bend left,out=-20,in=-160](u3) edge (u1)
			 ;
                \path[dotted,thick] (s1) edge (a1)
				     (s2) edge (a2) 
				     (s3) edge (a3); 
                \path[->]

                        (s3) edge (a1)

                        (t2) edge (a3)

                        (u3) edge (a1)
                        [bend left] (t3) edge (a1)
                        [out=-30,in=-145](t1) edge (a3)
                        [bend left,out=20,in=160]
                         (a1) edge (t1)
                         (a1) edge (u1)
                        [bend right,out=-20,in=-160]
                         (a2) edge (t2)
                         (a2) edge (u2)
                         (a3) edge (t3)
                         (a3) edge (u3)
                         [bend left,out=15,in=176](u2) edge (a3)
                         [bend right,out=0,in=-176](s1) edge (a3)
                         [bend right,out=0,in=-172](s2) edge (a3)
                         [bend right,out=5,in=-175](u1) edge (a3)
                         [bend right,out=-3,in=-180](u1) edge (a2)
                         [bend right,out=0,in=-170](u2) edge (a1)
                         [bend right,out=0,in=-175](t2) edge (a1)
                         [bend left,out=0,in=177](s2) edge (a1)
                         [bend left,out=10,in=160](s3) edge (a2)
                         [bend left,out=10,in=160](t3) edge (a2)
                         [bend left,out=5,in=170](u3) edge (a2)
                         [bend right,out=-10,in=-160](s1) edge (a2)
                         [bend right,out=-10,in=-160](t1) edge (a2)
                        ;
\end{tikzpicture}
\lvspace{-3mm}
\caption{Guessing the orientation of non-explicit conflicts is not enough.}
\label{fig:more_than_orientation}
\lvspace{-1mm}
\end{figure}

This is another indicator, yet far from a proof,
that the EC-conjecture holds and
by that 
Theorem~\ref{thm:stb_c-signature} describes the exact characterization of
the c-signature under stable semantics.

\lvspace{2mm}
\mysection{Discussion}

We introduced and studied the novel class of $\sigma$-compact argumentation frameworks for $\sigma$ among naive, stable, stage, semi-stable and preferred semantics.
We provided the full relationships between these classes, and showed that the extension verification problem is still $\coNP$-hard for stage, semi-stable and preferred semantics.
We next addressed the question of compact realizability: 
Given a set of extensions, is there a compact {\sc af} with this set of extensions under semantics $\sigma$?
Towards this end, we first used and extended recent results on maximal numbers of extensions to provide shortcuts for showing non-realizability.
Lastly we studied signatures, sets of compactly realizable extension-sets, and provided sufficient conditions for compact realizability.
This culminated in the explicit-conflict conjecture, a deep and interesting question in its own right:
Given an {\sc af}, can all implicit conflicts be made explicit?

Our work bears considerable potential for further research.
First and foremost, the explicit-conflict conjecture is an interesting research question.
But the EC-conjecture (and compact {\sc af}s in general) should not be mistaken for a mere theoretical exercise.
There is a fundamental computational significance to compactness:
When searching for extensions, arguments span the search space, since extensions are to be found among the subsets of the set of all arguments.
Hence the more arguments, the larger the search space.
Compact {\sc af}s are argument-minimal since none of the arguments can be removed without changing the outcome, thus leading to a minimal search space.
The explicit-conflict conjecture plays a further important role in this game:
implicit conflicts are something that {\sc af} solvers have to deduce on their own, paying mostly with computation time.
If there are no implicit conflicts in the sense that all of them have been made explicit, solvers have maximal information to guide search.



\smaller


\end{document}